\documentclass{article} % For LaTeX2e
\usepackage{iclr2025_conference,times}

\usepackage{amsmath,amsfonts,bm}

\def\eqref#1{equation~\ref{#1}}
\def\1{\bm{1}}

\DeclareMathAlphabet{\mathsfit}{\encodingdefault}{\sfdefault}{m}{sl}
\SetMathAlphabet{\mathsfit}{bold}{\encodingdefault}{\sfdefault}{bx}{n}

\usepackage{hyperref}
\usepackage{url}
\usepackage[utf8]{inputenc} % allow utf-8 input
\usepackage[T1]{fontenc}    % use 8-bit T1 fonts
\usepackage{hyperref}       % hyperlinks
\usepackage{url}            % simple URL typesetting
\usepackage{booktabs}       % professional-quality tables
\usepackage{amsfonts}       % blackboard math symbols
\usepackage{nicefrac}       % compact symbols for 1/2, etc.
\usepackage{microtype}      % microtypography
\usepackage{xcolor}         % colors

\usepackage{microtype}
\usepackage{graphicx}
\usepackage{subfigure}
\usepackage{svg}
\usepackage{enumitem}
\usepackage{algorithm}
\usepackage{algorithmic}
\usepackage{amsmath}
\usepackage{amssymb}
\usepackage{mathtools}
\usepackage{amsthm}
\usepackage{pifont}
\newcommand{\cmark}{\text{\ding{51}}}%
\newcommand{\xmark}{\text{\ding{55}}}%

\usepackage[capitalize,noabbrev]{cleveref}

\usepackage{tikz}
\usepackage{xcolor}
\definecolor{new_blue}{HTML}{D4E1F5}

\usetikzlibrary{decorations.pathreplacing,calc}
\newcommand{\tikzmark}[1]{\tikz[overlay,remember picture] \node (#1) {};}
\newcommand*{\AddNote}[4]{%
    \begin{tikzpicture}[overlay, remember picture]
        \draw [decoration={brace,amplitude=0.5em},decorate,thick,black]
            ($(#3)!(#1.north)!($(#3)-(0,1)$)$) --  
            ($(#3)!(#2.south)!($(#3)-(0,1)$)$)
                node [align=center, text width=2.5cm, pos=0.5, anchor=west] {#4};
    \end{tikzpicture}
}%

\usepackage{thmtools, thm-restate}
\theoremstyle{plain}
\newtheorem{theorem}{Theorem}[section]

\theoremstyle{definition}
\newtheorem{definition}[theorem]{Definition}

\theoremstyle{remark}

\usepackage[textsize=tiny]{todonotes}
\usepackage{xcolor}

\title{Differentially Private Deep Model-Based Reinforcement Learning}

\author{Alexandre Rio, Merwan Barlier, Igor Colin \& Albert Thomas\\
Noah's Ark Lab Paris, Huawei Technologies France\\
Correspondence to: \texttt{alexandre.rio2@huawei.com}
}

\iclrfinalcopy % Uncomment for camera-ready version, but NOT for submission.
\begin{document}

\maketitle

\begin{abstract}
We address private deep offline reinforcement learning (RL), where the goal is to train a policy on standard control tasks that is differentially private (DP) with respect to individual trajectories in the dataset. To achieve this, we introduce \textsc{PriMORL}, a model-based RL algorithm with formal differential privacy guarantees.
\textsc{PriMORL} first learns an ensemble of trajectory-level DP models of the environment from offline data.
It then optimizes a policy on the penalized private model, without any further interaction with the system or access to the dataset. 
In addition to offering strong theoretical foundations, we demonstrate empirically that \textsc{PriMORL} enables the training of private RL agents on offline continuous control tasks with deep function approximations, whereas current methods are limited to simpler tabular and linear Markov Decision Processes (MDPs). We furthermore outline the trade-offs involved in achieving privacy in this setting.
\end{abstract}

\section{Introduction}

%%%%%%%%%%%%%% NEW INTRO %%%%%%%%%%%%%%%%%%%%%

Despite Reinforcement Learning's (RL) notable advancements in various tasks, there have been many obstacles to its adoption for the control of real systems in the industry. In particular, online interaction with the system may be impractical or hazardous in real-world scenarios. 
%Offline RL \citep{Levine_Offline_RL_2020} intends to address this issue by enabling the training of control agents from static datasets.
Offline RL \citep{Levine_Offline_RL_2020} refers to the set of methods enabling the training of control agents from static datasets.
%It has therefore gained considerable traction in the reinforcement learning community, leading to the development of algorithms that effectively handle the challenge of limited exploration in the state-action space, a common issue causing traditional RL methods to fail.
While this paradigm shows promise for real-world applications, its deployment is not without concerns. Many studies have warned of the risk of privacy leakage when deploying machine learning models, as these models can memorize part of the training data. For instance, \citet{RigakiAttacksSurvey2020} review the proliferation of sophisticated privacy attacks.
Of the various attack types, membership inference attacks \citep{Shokri2017} stand out as the most prevalent. In these attacks, the adversary, with access to a black-box model trainer, attempts to predict whether a specific data point was part of the model's training data. 
Unfortunately, reinforcement learning is no exception to these threats. In a recent contribution, \citet{gomrokchi_membership_2022} exploit the temporal correlation of RL samples to perform powerful membership inference attacks using convolutional neural classifiers. More precisely, they demonstrate that given access to the output policy, 
%as well as black-box data oracle and deep RL trainer \albert{check black-box data oracle? deep RL trainer? not clear for me}, 
an adversary can learn to infer the presence of a specific trajectory --- which is the result of a sequence of interactions between a user and the system --- in the training dataset with great accuracy.
The threat of powerful membership inference attacks is particularly concerning in reinforcement learning, where a trajectory can unveil sensitive user information. For instance, when using RL to train autonomous vehicles \citep{ autonomous_driving_survey}, we need to collect a large number of trips that may disclose locations and driving habits. Similarly, a browsing journey collected to train a personalized recommendation engine may contain sensitive information about the user’s behavior \citep{news_recommendation}. In healthcare, RL's potential for personalized treatment recommendation \citep{ rl_treatment_recommendation} underscores the need to safeguard patients' treatment and health history.

A large body of work has focused on protecting against privacy leakages in machine learning. Differential Privacy (DP), which allows learning models without exposing sensitive information about any particular user in the training dataset, has emerged as the gold standard. While successfully applied in various domains, such as neural network training \citep{Abadi_2016} and multi-armed bandits \citep{TossouD16}, extending differential privacy to reinforcement learning poses challenges. In particular, the many ways of collecting data and the correlated nature of training samples resulting from online interactions make it difficult to come up with a universal and meaningful DP definition in this setting. Several attempts, such as joint differential privacy \citep{vietri2020private}, have been made, but extending definitions from bandits, they do not scale to large state and action spaces. Indeed, their scope is limited to tabular and linear Markov decision processes (MDPs) with finite horizon, which make them not suitable to the tasks typically encountered in deep RL.

In addition to its practical significance, the offline RL setting arguably offers a more natural framework for privacy. In contrast to online RL, which inherently blends input and output data throughout the process, an offline RL method can be seen as a black-box randomized algorithm $h$ taking in as input a fixed dataset $\mathcal{D}$, partitioned in trajectories, and outputting a policy $\hat{\pi}$. An adversary having access to $\hat{\pi}$ may successfully learn to infer the membership of a specific trajectory in $\mathcal{D}$, which can, as emphasized before, reveal sensitive user information. Hence, similarly to \citet{qiao_offline_2022}, we use the following informal DP definition for offline RL, which we refer to as \textit{trajectory-level differential privacy} (TDP): adding or removing a single trajectory from the input dataset of an offline RL algorithm must not impact significantly the distribution of the output policy. % Change reference for citation
If \citet{qiao_offline_2022} have proposed the first private algorithms for offline RL, building on value iteration methods, their approach is also restricted to finite-horizon tabular and linear MDPs, limiting its applicability. It is not suited for standard control tasks such as those from Gym \citep{openai_gym} and the DeepMind Control Suite \citep{DM_Control_Suite}, which often require deep neural function approximations. This leaves a huge gap between the current private RL literature and real-world applications. In this work, we are, according to our knwoledge, the first to tackle deep RL tasks in the infinite-horizon discounted setting under differential privacy guarantees, paving the way for enhanced applications of private RL in more complex scenarios.

\paragraph{Contributions.}

While previous work in the differentially private RL literature is essentially restricted to finite-horizon tabular and linear Markov decision processes (MDPs), with experiments reduced to simple numerical simulations, this work is the first attempt to tackle deep RL problems in the infinite-horizon discounted setting.
%, paving the way for enhanced applications of private RL in real-world scenarios.
To this end, we use a model-based approach, named \textsc{PriMORL}, which exploits a model of the environment to generalize the information contained in the offline data to unexplored regions of the state-action space. We introduce a method for training an ensemble of models with differential privacy guarantees at the trajectory-level, and mitigate the increased model uncertainty during model-based policy optimization. In addition to offering strong theoretical foundations and formal privacy guarantees, we show empirically that \textsc{PriMORL} can train private policies with competitive privacy-performance trade-offs on standard continuous control benchmarks, demonstrating the potential of our approach.
%\albert{this is only defined later, ideally we would find another way to say this here}. 
%To the best of our knowledge, this is the first private model-based method for offline RL. 

%\paragraph{Structure of the paper.}

%After discussing related work in Section \ref{sec:related_work}, we introduce the main notations and concepts from offline model-based RL and differential privacy in Section \ref{sec:preliminaries}. Then, in Section \ref{sec:dp_ombrl}, we introduce \textsc{DP-MORL}, a method to learn trajectory-level differentially private model-based RL policies from offline data. In Section \ref{sec:exps}, we assess our approach empirically on several continuous control tasks: \textsc{CartPole-Balance}, \textsc{CartPole-SinwgUp} and \textsc{HalfCheetah}, showing limited performance cost compared to non-private baselines. We later explore the privacy cost in offline RL (Section \ref{sec:price_of_privacy}) and address the broader implications of our work (Section \ref{sec:discussion}).
\section{Related Work}
\label{sec:related_work}

Offline RL \citep{Levine_Offline_RL_2020, Figueirido2022} focuses on training agents without further interactions with the system, making it essential in scenarios where data collection is impractical \citep{SinghKSRobotic22, SiqiHealthcare2020, autonomous_driving_survey}. Model-based RL \citep{MoerlandMBRLSurvey2023} can further reduce costs or safety risks by using a learned environment model to simulate beyond the collected data and improve sample efficiency \citep{ChuaCML18}. \citet{ArgensonD21} demonstrate that model-based offline planning, where the model is trained on a static dataset, performs well in robotic tasks. However, offline RL faces challenges like \textit{distribution shift} \citep{FujimotoMP19}, where the limited coverage of the dataset can lead to inaccuracies in unexplored state-action regions, affecting performance. Methods like \textsc{MOPO} \citep{yu_mopo_2020}, \textsc{MOReL} \citep{kidambi_morel_2021}, and \textsc{Count-MORL} \citep{kim_count_2023} address this by penalizing rewards based on model uncertainty, achieving strong results on offline benchmarks. Still, key design choices in offline MBRL require further exploration, as highlighted by \cite{mbrl_uncertainty_2022}.

On the other hand, Differential Privacy (DP), established by \citet{Dwork06}, has become the standard for privacy protection. Recent research has focused on improving the privacy-utility trade-off, with relaxations of DP and advanced composition tools enabling tighter privacy analyses \citep{DworkRV10, DworkR16, BunCDP2016, MironovRDP2017}. Notably, \textsc{DP-SGD} \citep{Abadi_2016} has facilitated the development of private deep learning algorithms, despite ongoing practical challenges \citep{ponomareva_how_2023}. Concurrently, sophisticated attack strategies have underscored the necessity for robust DP algorithms \citep{RigakiAttacksSurvey2020}. Recent studies have shown that reinforcement learning (RL) is also vulnerable to privacy threats \citep{pan_how_2019, prakash_how_2022, gomrokchi_membership_2022}. As RL is increasingly applied in personalized services \citep{Rl_perso_survey_2020}, the need for privacy-preserving training techniques is critical. Although DP has been successfully extended to multi-armed bandits \citep{TossouD16, Basu2019}, existing RL algorithms (\textit{e.g.}, \citet{vietri2020private}, \citet{zhou_differentially_2022}, \citet{qiao_tabular_2023}) with formal DP guarantees mainly apply to episodic tabular or linear MDPs and lack empirical validation beyond basic simulations. 
Moreover, private offline RL remains underexplored. Only \citet{qiao_offline_2022} have proposed DP offline algorithms, which, while theoretically strong, are also restricted to finite-horizon tabular and linear MDPs. Consequently, no existing work has introduced DP methods that can handle deep RL environments in the infinite-horizon discounted setting, a critical step toward deploying private RL algorithms in real-world applications. With this work, we aim to fill this gap by proposing a differentially private, deep model-based RL method for the offline setting.
\section{Preliminaries}
\label{sec:preliminaries}

\subsection{Offline Model-Based Reinforcement Learning}
\label{sec:offline_mbrl}

We consider an infinite-horizon discounted MDP, that is a tuple $\mathcal{M} = \left(\mathcal{S}, \mathcal{A}, P, r, \gamma, \rho_0 \right)$ where $\mathcal{S}$ and $\mathcal{A}$ are respectively the state and action spaces, $P: \mathcal{S} \times \mathcal{A} \longrightarrow \Delta\left(\mathcal{S}\right)$ the transition dynamics (where $\Delta\left(\mathcal{X}\right)$ denotes the space of probability distributions over $\mathcal{X}$), $r: \mathcal{S} \times \mathcal{A} \longrightarrow [0, 1]$ the reward function, $\gamma \in [0, 1)$ a discount factor and $\rho_0 \in \Delta\left(\mathcal{S}\right)$ the initial state distribution. The dynamics satisfy the Markov property, \textit{i.e.}, the next state $s^\prime$ only depends on current state and action. The goal is to learn a policy $\pi: \mathcal{S} \longrightarrow \Delta\left(\mathcal{A}\right)$ maximizing the expected discounted return $ \eta_\mathcal{M}(\pi) := \mathop{\mathbb{E}}_{\tau \sim \pi, \mathcal{M}}\left[R(\tau)\right]$, where $R(\tau) = \sum_{t=0}^\infty \gamma^t r_t$. The expectation is taken w.r.t. the trajectories $\tau = \left((s_t, a_t, r_t)\right)_{t \ge 0}$ generated by $\pi$ in the MDP $\mathcal{M}$, \textit{i.e.}, $s_0 \sim \rho_0$, $s_{t+1} \sim P(\cdot \vert s_t, a_t)$ and $a_t \sim \pi(\cdot | s_t)$. 
%The state value function $V^{\pi}_\mathcal{M}(s)$ and the Q-function $Q^{\pi}_\mathcal{M}(s, a)$ are the expected discounted return conditionally on starting from state $s$ and starting from state-action pair $(s, a)$, respectively.
%The state value function $\displaystyle V^{\pi}_M(s) := \mathop{\mathbb{E}}_{\tau \sim \pi, M} \left[R(\tau) \vert s=s_0\right]$ and the Q-function $\displaystyle Q^{\pi}_M(s, a) := \mathop{\mathbb{E}}_{\tau \sim \pi, M} \left[R(\tau) \vert s=s_0, a=a_0\right]$ are the expected discounted return when starting from state $s$ and starting from state-action pair $(s, a)$, respectively.

In offline RL, we assume access to a dataset of $K$ trajectories $\mathcal{D}_K = \left(\tau_k\right)_{k=1}^K$, where each $\tau_k = (s_t^{(k)}, a_t^{(k)}, r_t^{(k)})_{t \ge 0}$ has been collected with an unknown behavioral policy $\pi^B$. $\tau_k$ can be seen as the result of the interaction of a user $u_k$ with the environment. The objective is then to learn a policy $\hat{\pi}$ from $\mathcal{D}_K$ (without any further interaction with the environment) which performs as best as possible in $\mathcal{M}$. To achieve this goal, we consider a model-based approach. In this context, we learn estimates of both the transition dynamics and the reward function, denoted $\hat{P}$ and $\hat{r}$ respectively, from the offline dataset $\mathcal{D}_K$. This results in an estimate of the MDP $\hat{\mathcal{M}} = (\mathcal{S}, \mathcal{A}, \hat{P}, \hat{r}, \gamma, \rho_0)$. We can then use the model $\hat{\mathcal{M}}$ as a simulator of the environment to learn a policy $\hat{\pi}_{\hat{\mathcal{M}}}$, without further access to the dataset or interactions with the real environment modeled by $\mathcal{M}$. Note that if the policy $\hat{\pi}_{\hat{\mathcal{M}}}$ is trained to maximize the expected discounted return in the MDP model $\hat{\mathcal{M}}$, \textit{i.e.}, $\hat{\pi}_{\hat{\mathcal{M}}} \in \text{arg}\!\max \eta_{\hat{\mathcal{M}}}(\pi)$, we eventually want to evaluate the policy in the true environment $\mathcal{M}$, that is using $\eta_{\mathcal{M}}$.

\subsection{Differential Privacy}
\label{sec:diff_privacy}

When learning patterns from a dataset, differential privacy \citep{Dwork06} protects against the leakage of sensitive information in the data by ensuring that the output of the algorithm does not change significantly when adding or removing a data point, as formally stated in Definition \ref{def:dp}.
\begin{definition} \label{def:dp} \emph{$(\epsilon, \delta)$-differential privacy.}
    Given $\epsilon > 0$, $\delta \in [0, 1)$, a \textit{mechanism} $h$ (\textit{i.e.}, a randomized function of the data) is $(\epsilon, \delta)$-DP if for any pair of datasets $D$, $D^\prime$ that differ in a most one element (referred to as \textit{neighboring datasets}, and denoted $d(D, D^\prime) = 1$), and any subset $\mathcal{E}$ in $h$'s range:
    \[
        \mathbb{P} \left(h(D) \in \mathcal{E} \right) \le e^\epsilon \cdot \mathbb{P} \left(h(D^\prime) \in \mathcal{E} \right) + \delta \enspace .
    \]
\end{definition}
In particular, $\epsilon$ controls the strength of the privacy guarantees with larger $\epsilon$, which increases as $\epsilon$ grows.
%A small $\epsilon$ ensures that $h$'s output remains close between two neighboring datasets, the corresponding probability ratio being bounded by $\exp{(\epsilon)}$ when $\delta=0$. When $\delta$ is non-zero, this bound may only hold with probability greater than $1 - \delta$: we talk about \textit{approximate} DP, as opposed to \textit{pure} DP.
To achieve $\left(\epsilon, \delta\right)$-DP, the standard approach is to add a zero-mean random noise to the output of the (non-private) function $f$, whose magnitude $\sigma$ scales with $\Delta_\ell(f) / \epsilon$, where $\Delta_\ell(f):=\max\limits_{d(D, D^\prime)=1} \lVert f(D) - f(D^\prime) \rVert_\ell$ is the sensitivity of $f$.
%of the noise should be large enough to cover the largest amount by which the function $f$ can change between two neighboring datasets, \textit{i.e.}, the sensitivity $\Delta_\ell(f):=\max\limits_{d(D, D^\prime)=1} \lVert f(D) - f(D^\prime) \rVert_\ell$. $\sigma$ therefore scales with $\Delta_\ell(f) / \epsilon$.
%Differential privacy is thus a worst-case guarantee. However, too much noise may dramatically hurt the utility of the mechanism. This is why $\sigma$ typically scales with $\Delta_\ell(f) / \epsilon$, where $\epsilon$ controls the strength of the privacy guarantees and hence the trade-off between privacy and utility.
One of the most used DP mechanisms is the \textit{Gaussian mechanism}, which provably guarantees $(\epsilon, \delta)$-DP for $\epsilon, \delta \in (0,1)$ by adding random noise from a Gaussian distribution with magnitude $\displaystyle \sigma = \epsilon^{-1} \sqrt{2 \log{(1.25/\delta)}} \cdot \Delta_2(f)$. From such simple mechanisms, we can derive complex DP algorithms using the \textit{sequential} and \textit{parallel composition} properties of DP, as well as its \textit{immunity to post-processing} (\textit{i.e.}, if $h$ is $(\epsilon, \delta)$-DP and $g$ is data-independent, then $g \circ h$ remains $(\epsilon, \delta)$-DP). 

The Gaussian mechanism is central to \textsc{DP-SGD} \citep{Abadi_2016}, a learning algorithm that modifies classic SGD to ensure (approximate) differential privacy. By adding Gaussian noise to the gradients and bounding their norm by a constant $C$, \textsc{DP-SGD} enables private neural network training \citep{ponomareva_how_2023}. To track the total privacy budget $\epsilon_{\text{tot}}$ spent by \textsc{DP-SGD}, \citet{Abadi_2016} developed the \textit{moments accounting} method that provides a $\left(\mathcal{O}(q \epsilon \sqrt{T}), \delta\right)$-DP guarantee, where $q$ is the sampling ratio, $T$ is the number of iterations, and $\epsilon$ is the privacy parameter. \textsc{DP-SGD} relies strongly on privacy amplification by sub-sampling \citep{balle_privacy_2018}.
%, and subsequent refinements have used tools like Rényi DP \citep{MironovRDP2017}. 
Studies have also analyzed error bounds for \textsc{DP-SGD} under various loss assumptions \citep{BassilyST14, KangLLW23}.
%Moreover, in terms of performance, \textsc{DP-SGD} has been shown to degrade the excess empirical and population risks of non-private SGD by no more than $\Tilde{\mathcal{O}} \left(\frac{\sqrt{p}}{\epsilon N}\right)$, where $p$ is the dimension of the weights and $N$ is the number of data points.

% Main concepts and definition

% DP-SGD
\section{Differentially Private Model-Based Offline Reinforcement Learning}
\label{sec:dp_ombrl}

\begin{figure*}[t]
\vskip 0.2in
\begin{center}
\centering{\includegraphics[width=\textwidth]{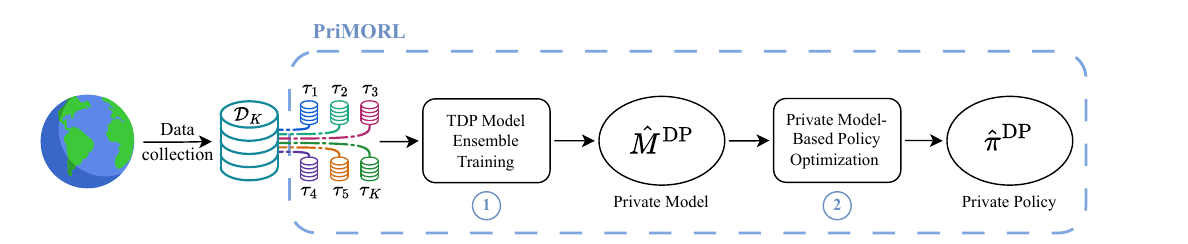}}
%\centering{\includegraphics[width=0.9\textwidth]{images/dp_morl_schema.png}}
\caption{\textsc{PriMORL} with its two main components: 1) private model training and 2) MBPO.}
\label{fig:offline_mbrl}
\end{center}
%\vskip -0.2in
\end{figure*}

We now describe our model-based approach for learning differentially private RL agents from offline data, which we call \textsc{PriMORL} (for \underline{Pri}vate \underline{M}odel-Based \underline{O}ffline \underline{RL}). After defining trajectory-level differential privacy (TDP) in offline RL (Section \ref{sec:dp_orl}), we address the learning of a private model from offline data (Section \ref{sec:dp_model_learning}). Finally, we demonstrate how we optimize a policy under the private model (Section \ref{sec:private_policy_opt}). Exploiting the \textit{post-processing} property of DP, we show that ensuring model privacy alone is enough to achieve a private policy. Figure \ref{fig:offline_mbrl} provides a high-level description of \textsc{PriMORL}.

\subsection{Trajectory-level Privacy in Offline Reinforcement Learning}
\label{sec:dp_orl}

We introduce the following formal definition for trajectory-level differential privacy (TDP) in offline RL. It can be seen as a reformulation of the definition used in \citet{qiao_offline_2022}, which is the first work to tackle differential privacy in this setting.

\begin{definition} \label{def:tdp} \emph{$(\epsilon, \delta)$-TDP}. Let $h$ be an offline RL algorithm, that takes as input an offline dataset and outputs a policy. Given $\epsilon > 0$ and $\delta \in (0,1)$, $h$ is $(\epsilon, \delta)$-TDP if for any trajectory-neighboring datasets $\mathcal{D}_K$, $\mathcal{D}_{K \setminus \{k\}}$,  and any subset of policies $\Pi$:
\[
    \mathbb{P} \left(h(\mathcal{D}_K) \in \Pi \right) \le e^{\epsilon} \cdot \mathbb{P} \left(h(\mathcal{D}_{K \setminus \{k\}}) \in \Pi \right) + \delta \enspace .
\]
    
\end{definition}

\subsection{Model Learning with Differential Privacy}
\label{sec:dp_model_learning}

%\textsc{PriMORL} starts by learning an $(\epsilon, \delta)$-TDP model of the environment. Our first contribution is to provide a training method that is tailored to trajectory-level privacy (Section~\ref{sec:tdp_model_training}). We then formally analyze the privacy guarantees of the model in Section~\ref{sec:model_privacy_guarantees}.

%\subsubsection{Dynamics Model Ensemble}

Following previous work \citep{yu_mopo_2020, kidambi_morel_2021}, we jointly model the transition dynamics $\hat{P}$ and reward $\hat{r}$ with a Gaussian distribution $\hat{M}$ conditioned on the current state and action. Its mean and covariance are parameterized with neural networks $\theta = \left(\phi, \psi\right)$:
\[
    \hat{M}_{\theta}\left(\Delta_{t}^{t+1}(s), r_t \vert s_t, a_t\right) = \mathcal{N} \left(\mu_{\phi}(s_t, a_t), \Sigma_{\psi}(s_t, a_t)\right) \enspace .
\]
To carry out uncertainty estimation (see Section \ref{sec:private_policy_opt}), we train an ensemble of $N$ models $\hat{M}_{\theta_i}$, $i \in [\![1, N]\!]$, all sharing the same architecture. The core aspect of \textsc{PriMORL}, as illustrated in Figure \ref{fig:offline_mbrl}, is therefore to learn a trajectory-level DP dynamics model ensemble. 

%A straightforward approach would consist in training each model independently and substituting the non-private optimizer with \textsc{DP-SGD}. However, such an approach would be very inefficient. Indeed, \textsc{DP-SGD} perturbs gradients at the transition level, causing excessive noise when aggregated at the trajectory level, ultimately undermining model performance. Additionally, since every model in the ensemble uses the same dataset $\mathcal{D}_K$, the privacy budget would scale with $N$ due to the sequential composition property of differential privacy, exacerbating privacy leakage. 
A straightforward approach would be to train each model independently using \textsc{DP-SGD}, but this is inefficient. \textsc{DP-SGD} introduces excessive noise by perturbing gradients at the transition level, harming model performance. Moreover, since all models use the same dataset $\mathcal{D}_K$, the privacy budget scales with $N$, increasing privacy leakage.
A key contribution of our work, developed in Section~\ref{sec:tdp_model_training}, is thus to introduce a training method that 1) ensures privacy guarantees at the trajectory level and 2) efficiently manages the privacy budget across an ensemble of models. 
%Note that we learn to predict the consecutive state difference $\Delta_{t}^{t+1}(s) = s_{t+1} - s_t$ instead of next state $s_{t+1}$, in order to guarantee continuity, which is standard in model-based RL. Moreover, instead of training a single dynamics model, we train an ensemble of $N$ models $\left\{\hat{M}_{\theta_i}\right\}_{i=1}^N$, all sharing the same architecture. This enables us to carry out uncertainty estimation (see Section \ref{sec:private_policy_opt}).

\subsubsection{Trajectory-level DP Training for Model Ensembles} \label{sec:tdp_model_training}

%Such a model is typically trained on the offline dataset $\mathcal{D}_K$ using first-order optimizers like SGD or Adam. The core aspect of our approach, as illustrated in Figure \ref{fig:offline_mbrl}, is to learn a private model.

%To achieve this, we replace the standard optimizer with a DP optimizer. A natural choice is to use \textsc{DP-SGD}. However, \textsc{DP-SGD} is not the most suitable to handle trajectory-level privacy, especially as it clips per-sample gradients. 
%Instead, we use the DP training methods developed in \citet{McMahanRT018}, namely \textsc{DP-FedSGD} and \textsc{DP-FedAvg} especially designed for user-level privacy (equivalent to trajectory-level privacy in our setting), and originally intended for recurrent language models.
We identified that the idea behind the DP training method developed in \citet{McMahan_FedAvg_17}, \textsc{DP-FedAvg}, although originally designed to achieve client-level privacy in federated settings, could be effectively adapted for trajectory-level private training in offline RL. Specifically, our training data can be partitioned into trajectories in a manner analogous to how data is partitioned across clients in federated learning. This insight allowed us to leverage this approach to address the unique privacy challenges of our task, adapting it to model ensembles.

%Instead, we use the DP training method developed in \citet{McMahanRT018}, namely \textsc{DP-FedAvg}. While this algorithm based on \textsc{FedAvg} \citep{McMahan_FedAvg_17} is originally intended to achieve client-level privacy in federated settings, we note that it can be used in any context where the training data can be partitioned. This is the case in offline RL, where the data can be segmented by trajectory, making \textsc{DP-FedAvg} an interesting approach to training a model with trajectory-level privacy.

%The core idea behind \textsc{DP-FedSGD} and \textsc{DP-FedAvg} is to draw, at each iteration $t$, a random subset $\mathcal{U}_t$ of the $K$ trajectories.
We present the resulting training procedure in Algorithm~\ref{alg:dp_model_training}.
%\footnote{Further details about the implementation and the training of the model are provided in appendix}. 
The core idea behind \textsc{TDP Model Ensemble Training} is to draw, at each iteration $t$, a random subset $\mathcal{U}_t$ of the $K$ trajectories (line 2).
Each trajectory is drawn with probability $q$, so that the expected number of trajectories selected at each step is $qK$. For each trajectory $\tau_k \in \mathcal{U}_t$, the clipped gradients $\{\Delta_{i, k}^{\text{clipped}}(t)\}$ are then computed from $\tau_k$'s data only (line 3 to 7). Processing and clipping gradients per trajectory is essential to provide trajectory-level privacy, as this ensures that no trajectory will carry more weight than another in the optimization of the model. We later introduce ensemble-adapted clipping strategies to control the privacy budget over model ensembles, ensuring that the sensitivity of the ensemble gradient $\Delta_k^{\text{clipped}}(t) = \left(\Delta_{i, k}^{\text{clipped}}(t)\right)_{i=1}^N$ is bounded by $C$.
%This ensures that no trajectory will carry more weight than another in the optimization of the model, hence preventing privacy leakage from a specific trajectory.
We then compute an unbiased estimator of the subset gradient average whose sensitivity is bounded by $C / qK$ (line 8). We can then apply the Gaussian mechanism with magnitude $\sigma = zC/qK$, where $z$ controls the strength of the privacy guarantee $\epsilon$, and update the ensemble model $\theta(t) = \left(\theta_i(t)\right)_{i=1}^N$ with noisy gradient (line 9):
\[
    \theta(t+1) \longleftarrow \theta(t) + \Delta^{\text{avg}}(t) + \mathcal{N}\left(0_{Nd}, \sigma^2 I_{Nd}\right) \enspace .
\]
%While sharing the same steps as described above, \textsc{DP-FedSGD} and \textsc{DP-FedAvg} differ in how each trajectory gradient is computed.

%Algorithm \ref{alg:dp_model_training} provides a condensed pseudo-code for the above training procedure. Further details about the implementation and the training of the model are provided in appendix.

\begin{algorithm}[ht]
\begin{algorithmic}[1]
\FOR{each iteration $t \in [\![0, T-1]\!]$}
    \STATE $\mathcal{U}_t \leftarrow $ (sample with replacement trajectories from $\mathcal{D}_K$ with prob. $q$)
    \FOR{each trajectory $\tau_k \in \mathcal{U}_t$}
        \STATE Clone current models $\left\{\theta_i^{\text{start}}\right\}_{i=1}^N \leftarrow \left\{\theta_i(t)\right\}_{i=1}^N$
        \STATE $\left\{\theta_{i, k}\right\}_{i=1}^N \leftarrow \text{\textsc{EnsClipGD}}\left(\tau_k, \left\{\theta_i^{\text{start}}\right\}_{i=1}^N; C, \text{local epochs }E, \text{batch size }B\right)$
        \STATE $\Delta^{\text{clipped}}_{i,k}(t) \leftarrow \theta_{i, k} - \theta_i^{\text{start}}, \enspace i=1,...,N$
    \ENDFOR
    \STATE $\Delta^{\text{avg}}_i(t) = \frac{\sum_{k \in \mathcal{U}_t} \Delta_{i, k}^{\text{clipped}}(t)}{q K}, \enspace i=1,...,N$
    \STATE $\theta(t+1) \leftarrow \theta(t) +  \Delta^{\text{avg}}(t) + \mathcal{N}\left(0_{Nd}, \left(\frac{zC}{qK}\right)^2 I_{Nd}\right)$
\ENDFOR
\end{algorithmic}
\caption{\textsc{TDP Model Ensemble Training}}
\label{alg:dp_model_training}
\end{algorithm}

\subsubsection{Privacy Guarantees for the Model} \label{sec:model_privacy_guarantees}

We can now derive formal privacy guarantees for a model trained using Algorithm~\ref{alg:dp_model_training}. A key challenge in our setting arises from training an ensemble of \(N\) models for uncertainty estimation, all using the same dataset \(\mathcal{D}_K\). Treating each model independently, with separate clipping and noise addition, would be inefficient and significantly increase the privacy budget by composition. This could be mitigated by limiting the ensemble size, but at the cost of performance, as shown in \cite{mbrl_uncertainty_2022}.
%In this case, the noise would obscure the contribution of a given trajectory only to an individual model, and each training step would count as \(N\) separate queries to the private dataset. Due to the sequential composition property of differential privacy, the total privacy budget scales linearly with \(N\) in this case. 
%Although offline RL methods often use relatively small ensembles (\(N=4\) in \cite{kidambi_morel_2021} and \(N=7\) in \cite{yu_mopo_2020}), this could still lead to unreasonable privacy costs. Reducing the ensemble size might mitigate this issue, but it could also degrade performance, as shown in \cite{mbrl_uncertainty_2022}.

To address this challenge, we process all the gradients of the model ensemble simultaneously and distribute the global clipping norm $C$ across all models, on the same principle as the per-layer clipping used in \cite{McMahan_FedAvg_17}. Denoting $\Delta_{i, \ell}$ the gradient of layer $\ell$ for model $i$, we propose and experiment with two ensemble clipping strategies: \textbf{Flat Ensemble Clipping}, which clips the whole model gradient $\Delta_i = (\Delta_{i,\ell})_{\ell=1}^L$ with $C_i = C/\sqrt{N}$; and \textbf{Per Layer Ensemble Clipping}, which clips per-layer gradients $\Delta_{i, \ell}$ with $C_{i, \ell} = C /\sqrt{N \times L}$, so that $C = \sqrt{\sum_{i=1}^N C_i^2} = \sqrt{\sum_{i=1}^N \sum_{\ell=1}^L C_{i, \ell}^2}$. For both strategies, we verify that that $\Delta^\text{clipped}_{k} = \left(\Delta^\text{clipped}_{i, k}\right)_{i=1}^K$ has sensitivity bounded by $C$ (see Theorem~\ref{thm:dp_model_privacy}'s proof in appendix), and that the contribution of a given trajectory to the \textit{model ensemble} is appropriately limited. Ensemble clipping eliminates the linear dependence of the privacy budget on the number of models. However, it does entirely remove the negative impact of increasing \(N\). Indeed, for a given noise level, a larger \(N\) requires a smaller clipping threshold \(C_i\) or \(C_{i,\ell}\), which can degrade model convergence by losing too much information from the original gradient. Nevertheless, the clipping threshold scales with the square root of \(N\), mitigating the impact to some extent.

We now formally derive the privacy guarantees for an ensemble of models trained with Algorithm~\ref{alg:dp_model_training}. Mapping users in federating learning to trajectories in offline RL, we can directly adapt Theorem 1 from \citet {McMahanRT018} to state that, with the sensitivity of clipped gradients $\Delta^\text{clipped}_{i,k}$ effectively bounded by $C$, the moments accounting method from \citet{Abadi_2016} computes correctly the privacy loss of Algorithm~\ref{alg:dp_model_training} at trajectory-level for the noise multiplier $z = \sigma / \mathbb{C}$ with $\mathbb{C} = C / qK$. We can therefore use the moments accountant to compute, given $\delta \in (0,1)$, $z>0$, $q \in (0,1)$ and $T \in \mathbb{N}$, the total privacy budget $\epsilon$ spent by Algorithm \ref{alg:dp_model_training}, and obtain $(\epsilon, \delta)$-TDP guarantees for our dynamics model, as stated in Theorem~\ref{thm:dp_model_privacy} (full proof in appendix).

\begin{restatable}{theorem}{DPModelPrivacy}\label{thm:dp_model_privacy} \emph{$(\epsilon, \delta)$-TDP guarantees for dynamics model}.
Given $\delta \in (0,1)$, noise multiplier $z$, sampling ratio $q$ and number of training iterations $T$, let $\epsilon := \epsilon^{\text{MA}}\left(z, q, T, \delta\right)$ be the privacy budget computed by the moments accounting method from \citep{Abadi_2016}. The dynamics model output by Algorithm \ref{alg:dp_model_training} is $(\epsilon, \delta)$-TDP.
\end{restatable}

\subsection{Policy Optimization under a Private Model}
\label{sec:private_policy_opt}

Now that we learned a private model $\hat{M}$ from offline data, we use it as a simulator of the environment to learn a private policy $\hat{\pi}$ with a model-based policy optimization approach. The use of a private model and the privacy constraints on the end policy introduce additional challenges compared to the non-private case, as demonstrated in Section~\ref{sec:impact_privacy_policy_opt}. We study solutions to mitigate the detrimental effects of private training on policy performance in Section~\ref{sec:solutions_private_privacy_opt}, before deriving formal privacy guarantees for a policy learned under a private model in Section~\ref{subsubsec:policy_opt}.

\subsubsection{Impact of Privacy on Policy Optimization} \label{sec:impact_privacy_policy_opt}

It is first essential to examine the complexities of policy optimization in model-based offline RL and assess whether they are amplified in the private setting.
A major challenge in model-based offline RL is to handle the discrepancy between the true and the learned dynamics when optimizing the policy. Indeed, model inaccuracies cause errors in policy evaluation that may be exploited, resulting in poor performance in the real environment. According to the Simulation Lemma \citep{Kearns2002, Xu2020}, the value evaluation error of a policy $\pi$ in model-based RL can be decomposed into a \textit{model error} term and a \textit{policy distribution shift} term. Formally, denoting $\rho_{\pi^B}$ the state-action discounted occupancy measure of the data-collection policy $\pi^B$, if the model error is bounded as $\mathbb{E}_{(s,a) \sim \rho_{\pi^B}}\left[D_{KL}\left(P(\cdot \vert s,a) \Vert \hat{P}(\cdot \vert s,a) \right) \right] \le \varepsilon_m$ and the distribution shift is bounded as $\max_{s} D_{KL} \left(\pi(\cdot \vert s) \Vert \pi^B(\cdot \vert s))\right) \le \varepsilon_p \enspace$, then the value evaluation error of $\pi$ is bounded as:
\begin{equation} \label{eq:simulation_lemma}
    \vert \hat{V}^\pi - V^\pi \vert \le \frac{\sqrt{2} \gamma}{(1 - \gamma)^2} \sqrt{\epsilon_m} + \frac{2 \sqrt{2}}{(1 - \gamma)^2} \sqrt{\varepsilon_p} \enspace,
\end{equation}
where $\hat{V}^\pi$ and $V^\pi$ denote the value of $\pi$ under the learned and the true dynamics, respectively.

Under some assumptions regarding the model loss function, we can refine the model error term in terms of the size $N$ of the training dataset, as stated in Proposition\ref{thm:perf_policy_non_private}.

\begin{restatable}{proposition}{PerfPolicyNonPrivate} \label{thm:perf_policy_non_private}
    \emph{Value evaluation error in non-private MBRL.} Let the model loss function be $L$-Lipschitz and $\Delta$-strongly convex, and assumptions from the simulation lemma hold. There are a stochastic convex optimization algorithm for learning the model and a constant $M$ such that, with probability at least $1-\alpha$, and for sufficiently large $N$, the value evaluation error of $\hat{\pi}$ is bounded as:
    \[
         \vert \hat{V}^{\hat{\pi}} - V^{\hat{\pi}} \vert \le \frac{\sqrt{2} \gamma}{(1 - \gamma)^2} \cdot M \cdot \frac{L \log^{1/2}(N / \alpha)}{\sqrt{\Delta N}} + \frac{2 \sqrt{2} \gamma}{(1 - \gamma)^2} \sqrt{\varepsilon_p} \enspace .
    \]
\end{restatable}

When we learn the model with differential privacy, we disrupt model convergence because of gradient clipping and noise. This likely results in a less accurate dynamics model (although it may help prevent overfitting in some cases) and increased value evaluation error. Intuitively, DP training impacts both error terms in \ref{eq:simulation_lemma}. Model error is affected as a direct result of gradient perturbations: \cite{BassilyST14}, in particular, shows that noisy gradient descent (GD) has increased excess risk compared to non-private GD. The effect of DP training on \textit{distribution shift} is more ambiguous, but as privacy intuitively aims at preventing overfitting particular data points, it is therefore likely to drive $\hat{\pi}$ away from the offline data, hence increasing the divergence between $\hat{\pi}$ and $\pi^B$. In the simpler case where the model is trained with a vanilla DP noisy GD algorithm, denoting $\varepsilon_p^{DP}$ the upper bound on $\max_{s} D_{KL} \left(\hat{\pi}^{DP}(\cdot \vert s) \Vert \pi^B(\cdot \vert s))\right)$, Proposition~\ref{thm:perf_policy} states the private value evaluation error.

\begin{restatable}{proposition}{PerfPolicy} \label{thm:perf_policy}
    \emph{Value evaluation error in private MBRL.} Let assumptions from Proposition~\ref{thm:perf_policy_non_private} hold. If the model is learned with $(\epsilon, \delta)$-DP gradient descent, then, with probability at least $1-\alpha$, there is a constant $M^\prime$ such that for large enough $N$, the value evaluation error of  $\hat{\pi}^{DP}$ is bounded as:
    \[
         \vert \hat{V}_\text{DP}^{\hat{\pi}} - V^{\hat{\pi}} \vert \le \frac{\sqrt{2} \gamma}{(1 - \gamma)^2} \cdot M^\prime \cdot \frac{L d^{1/4} \log(N / \delta) \cdot \text{poly} \log(1 / \alpha)}{\sqrt{\Delta N \epsilon \alpha}} + \frac{2 \sqrt{2} \gamma}{(1 - \gamma)^2} \sqrt{\varepsilon_p^{DP}} \enspace .
    \]
\end{restatable}

Comparing the value evaluation errors in Propositions~\ref{thm:perf_policy_non_private} and \ref{thm:perf_policy} (both proven in appendix), we observe how DP training may degrade performance in MBRL. On the model error term, the private bound has an explicit dependence on the problem dimension $d$ which is not present in the non-private bound, and the $\sqrt{\epsilon}$ factor in the denominator shows that the error will degrade with strong privacy guarantees. On the distribution shift term, $\varepsilon_p^{DP}$ is possibly greater than $\varepsilon_p$, as discussed above. Overall, DP training thus reduces the reliability of policy evaluation, leading to poorer policy performance in the true environment. This effect intensifies as privacy guarantees become stricter.

\subsubsection{Mitigating Private Model Uncertainty} \label{sec:solutions_private_privacy_opt}

Building on the previous analysis, we now study how to handle model uncertainty in the private model-based setting.
In the non-private case, the discrepancy between the true and the learned dynamics is typically handled by penalizing the reward with a measure of the uncertainty of the model, denoted $u: \mathcal{S} \times \mathcal{A} \rightarrow \mathbb{R}_+$. Therefore, if the model is believed to be unreliable at a given state-action pair $(s,a)$ (\textit{i.e.}, large $u(s,a)$), the possibly over-estimated reward will be corrected as:
\begin{equation} \label{eq:reward_penalty}
    \Tilde{r}(s,a) = \hat{r}(s,a) - \lambda \cdot u(s,a) \enspace ,
\end{equation}
where $\lambda$ is an hyperparameter. 
%Policy optimization under the resulting pessimistic MDP $\Tilde{\mathcal{M}} = (\mathcal{S}, \mathcal{A}, \hat{P}, \Tilde{r}, \gamma, \rho_0)$ ensures that the learned policy $\hat{\pi}_{\Tilde{\mathcal{M}}}$ should perform at least as well as the policy (or mixture of policies) used to collect the data (\citet{yu_mopo_2020}, Theorem 4.4).
The policy is the optimized under the resulting pessimistic MDP $\Tilde{\mathcal{M}} = (\mathcal{S}, \mathcal{A}, \hat{P}, \Tilde{r}, \gamma, \rho_0)$.
\textsc{MOPO} \citep{yu_mopo_2020}, \textsc{MOReL} \citep{kidambi_morel_2021} and more recently \textsc{Count-MORL} \citep{kim_count_2023} achieve impressive results on traditional offline RL benchmarks with this approach, using different heuristics to estimate model uncertainty.

%We believe this penalization approach will also work in the private setting, as long as we mitigate the enhanced model discrepancy resulting from DP training.
\cite{mbrl_uncertainty_2022}, which study design choices in offline model-based RL and the properties of different uncertainty estimators, find that the uncertainty measures proposed in the literature have a good correlation to model error, more so than with the distribution shift. Even if we suspect that it may negatively impact distribution shift, the obvious and quantifiable effect of DP training is on the dynamics error. Therefore, we believe that existing measures are appropriate for mitigating the increase in dynamics error from DP training. In particular, we consider the maximum aleatoric uncertainty $u_{\text{MA}}(s,a) = \max_{i \in [\![1, N]\!]} \Vert \Sigma_{\psi_i}(s,a) \Vert_F$ \citep{yu_mopo_2020} and the maximum pairwise difference $u_\text{MPD}(s,a) = \max_{i, j \in [\![1, N]\!} \Vert \mu_{\phi_i}(s,a) - \mu_{\phi_j}(s,a) \Vert_2$ \citep{kidambi_morel_2021}. We compare both estimators (see Figure~\ref{fig:diff_estimators} in the appendix) and find that neither is consistently superior. However, we observe that the choice of estimator can affect performance on a specific task.
In addition, it seems reasonable to increase the reward penalty $\lambda$ compared to the non-private case to take into account the greater uncertainty. If we cannot quantify how much we should increase $\lambda$, we try different values empirically and observe that if a too large $\lambda$ hurt policy performance, a moderate increase indeed provide better results.

%$u_{\text{MA}}(s,a)$ focuses more on aleatoric uncertainty, as it captures the variability in predictions across an ensemble of models due to inherent randomness in the environment. In contrast, $u_\text{MPD}(s,a)$ explicitly measures epistemic uncertainty by quantifying the disagreement between model predictions. DP training affects epistemic uncertainty, since it disrupts the convergence of the model. As a result, $u_\text{MPD}(s,a)$, which reflects this uncertainty, is a more suitable metric for capturing the negative effects of DP training on model performance. Additionally, DP training affects the aleatoric uncertainty estimates $\{\Sigma_{\psi_i}\}_{i=1}^N$, making them less representative of the true uncertainty of the environment. Therefore, $u_\text{MA}(s,a)$ becomes less useful for capturing the actual negative impact of DP training on model reliability. Finally, we consider that $u_\text{MPD}(s,a)$ is better suited for reward penalization in the private setting, and we observe its superiority in our experiments.

\subsubsection{Private Policy Optimization}
\label{subsubsec:policy_opt}

%Now that we assessed the suitability of $u_\text{MPD}$ to measure model uncertainty in private MBRL
Given a choice of uncertainty estimator $u \in \{u_\text{MA}, u_\text{MPD}\}$, we now consider optimizing the policy within the pessimistic private MDP $\Tilde{\mathcal{M}} = (\mathcal{S}, \mathcal{A}, \hat{P}, \Tilde{r}_u, \gamma, \rho_0)$, with $\Tilde{r}_u = \hat{r}(s,a) - \lambda \cdot u(s,a)$.
%and further prove formal  TDP guarantees for the resulting policy. 
We use Soft Actor-Critic (SAC, \citet{Haarnoja_SAC_2018})\footnote{This could be any model-based policy optimization or planning algorithm that does not use offline data.}, a classic off-policy algorithm with entropy regularization, to learn the policy from $\Tilde{\mathcal{M}}$, in line with existing approaches in the offline MBRL literature.
Offline model-based methods typically mix real offline data from $\mathcal{D}_K$ with model data during policy learning (in MOPO, for instance, each batch contains 5\% of real data). Here, however, we learn the policy exclusively from model data to avoid incurring privacy loss beyond what is needed to train the model, and thus control the privacy guarantees. Algorithm \ref{alg:sac_training} in appendix provides a pseudo-code for SAC policy optimization in the pessimistic private MDP. 
%\albert{tu penses qu'on aurait pu prendre des real data tant qu'on ne prend pas trop de trajectoires? par example 1 transition reelles par trajectoire?}
Using the post-processing property of DP, we can now state in Theorem~\ref{thm:dp_policy_privacy} that, given the $(\epsilon, \delta)$-TDP model $\hat{M} = (\hat{P}, \hat{r})$ learned as described in Section \ref{sec:dp_model_learning}, the policy learned with Algorithm~\ref{alg:sac_training} under $\Tilde{M}$ is also $(\epsilon, \delta)$-TDP. The full proof of this theorem is provided in appendix.
%First, we observe that since the covariance estimators $\{\Sigma_{\psi_i}\}_{i=1}^N$ are learned privately, the uncertainty estimator $u(s,a) = \lVert \Sigma_{\psi_i}(s,a) \rVert_F$ is also private thanks to the post-processing property of DP. Therefore, the pessimistic model $\hat{M}$ remains $(\epsilon, \delta)$-TDP.
%Now, we can think of SAC model-based policy optimization as an abstract, randomized function $h_\Pi$, that takes as input $\hat{M}$ and outputs as policy $\hat{\pi}$. Furthermore, let $h_M$ denote the mechanism that takes as input the private offline dataset $\mathcal{D}_K$ and outputs the private pessimistic model $\hat{M}$, and which is $(\epsilon, \delta)$-TDP following \ref{thm:dp_model_privacy}. We observe that $h = h_\Pi \circ h_M$, where $h$ is the global offline RL algorithm which is the object of Definition \ref{def:tdp}. Since SAC only uses data from the model, as stated in Section \ref{subsubsec:policy_opt}, $h_\Pi$ is independent of the private offline data $\mathcal{D}_K$. In other words, $h_\Pi$ is a data-independent transformation of the private mechanism $h_M$. Thanks again to the post-processing property of differential privacy, $h$ is also $(\epsilon, \delta)$-TDP.

\begin{restatable}{theorem}{DPPolicyPrivacy}\label{thm:dp_policy_privacy}\emph{$(\epsilon, \delta)$-TDP guarantees for \textsc{PriMORL}}.
    Given an $(\epsilon, \delta)$-TDP model $\left(\hat{P}, \hat{r}\right)$ learned with Algorithm \ref{alg:dp_model_training}, the policy obtained with private policy optimization (Algorithm~\ref{alg:sac_training}) within the pessimistic model $\left(\hat{P}, \hat{r} - \lambda u\right)$ is $(\epsilon, \delta)$-TDP.
\end{restatable}

%\textsc{PriMORL} is thus able to learn trajectory-level differentially private policies from offline data. % name the algorithm before

\section{Experiments}
\label{sec:exps}

We empirically assess \textsc{PriMORL} in three continuous control tasks: \textsc{CartPole-Balance} and \textsc{CartPole-Swingup} from the DeepMind Control Suite \citep{DM_Control_Suite} as well as \textsc{Pendulum} from OpenAI's Gym \citep{openai_gym}. We also conduct experiments on \textsc{HalfCheetah} \citep{halfcheetah}, which we present in appendix (Section~\ref{sec:exps_halfcheetah}). For simplicity, we refer to \textsc{CartPole-Balance} and \textsc{CartPole-Swingup} as \textsc{Balance} and \textsc{Swingup}.
%We refer the reader to the appendix for a detailed presentation of the tasks.

\subsection{Experimental Setting}

Following common practice, we evaluate the offline policies by running them in the real environment %(\textit{i.e.}, online policy evaluation). 
We aim to assess the policy's performance degradation when varying the privacy level, as DP training may negatively affect it. 
%As there is no benchmark for our \textsc{CartPole} datasets, we also report results with the online model-free method \textsc{DDPG} \citep{DDPG} to ensure that we actually learn good policies from our offline dataset.
We consider \textsc{MOPO} as our non-private baseline. For \textsc{PriMORL}, we consider different configurations outlined in Table~\ref{tab:dp_morl_config}. The \textsc{No Privacy} variant, without noise ($z=0$), isolates the impact of trajectory-level model training on performance.
%They are not private ($\epsilon=\infty$) but allow to assess the impact of these components on model training. 
%We propose \textsc{PriMORL No Clip}, our method without gradient clipping and without noise, to isolate trajectory-level model optimization, while \textsc{PriMORL No Noise}, our method with gradient clipping and without noise, allows to evaluate the impact of gradient clipping. 
%For \textsc{CartPole-SwingUp}, we provide three private configurations ($\epsilon < \infty$): \textsc{PriMORL Medium Privacy} corresponds to a smaller noise multiplier $z=0.25$ (and thus a higher $\epsilon$), while \textsc{PriMORL High Privacy} corresponds to a higher noise multiplier $z=0.3$ (and thus a smaller $\epsilon$).
The two private variants ($\epsilon < \infty$):  \textsc{PriMORL Low} and \textsc{PriMORL High} correspond to different noise multipliers. We detail the choice of privacy parameters in appendix (Section~\ref{sec:detail_exps}). As the existing \textsc{SwingUp} offline benchmark from \citet{RL_unplugged} is very small ($K=40$), and DP training of ML models typically requires significantly more data compared to non-private training (see, for instance, \citet{ponomareva_how_2023}, and our discussion in appendix, Section~\ref{sec:app_price_privacy}), we build our own dataset with $30$k trajectories (\textit{i.e.}, $30$M steps). We follow the same approach for \textsc{Balance} and \textsc{Pendulum} for which we are not aware of any existing offline benchmark. Data collection, which we detail in appendix (Section~\ref{sec:app_data_collection}), follows the philosophy of standard benchmarks like \textsc{D4RL} \citep{fu2020d4rl}.

\begin{figure*}[t]
\centering
\begin{minipage}{.31\textwidth}
%\vspace{0pt}
\includegraphics[width=\linewidth]{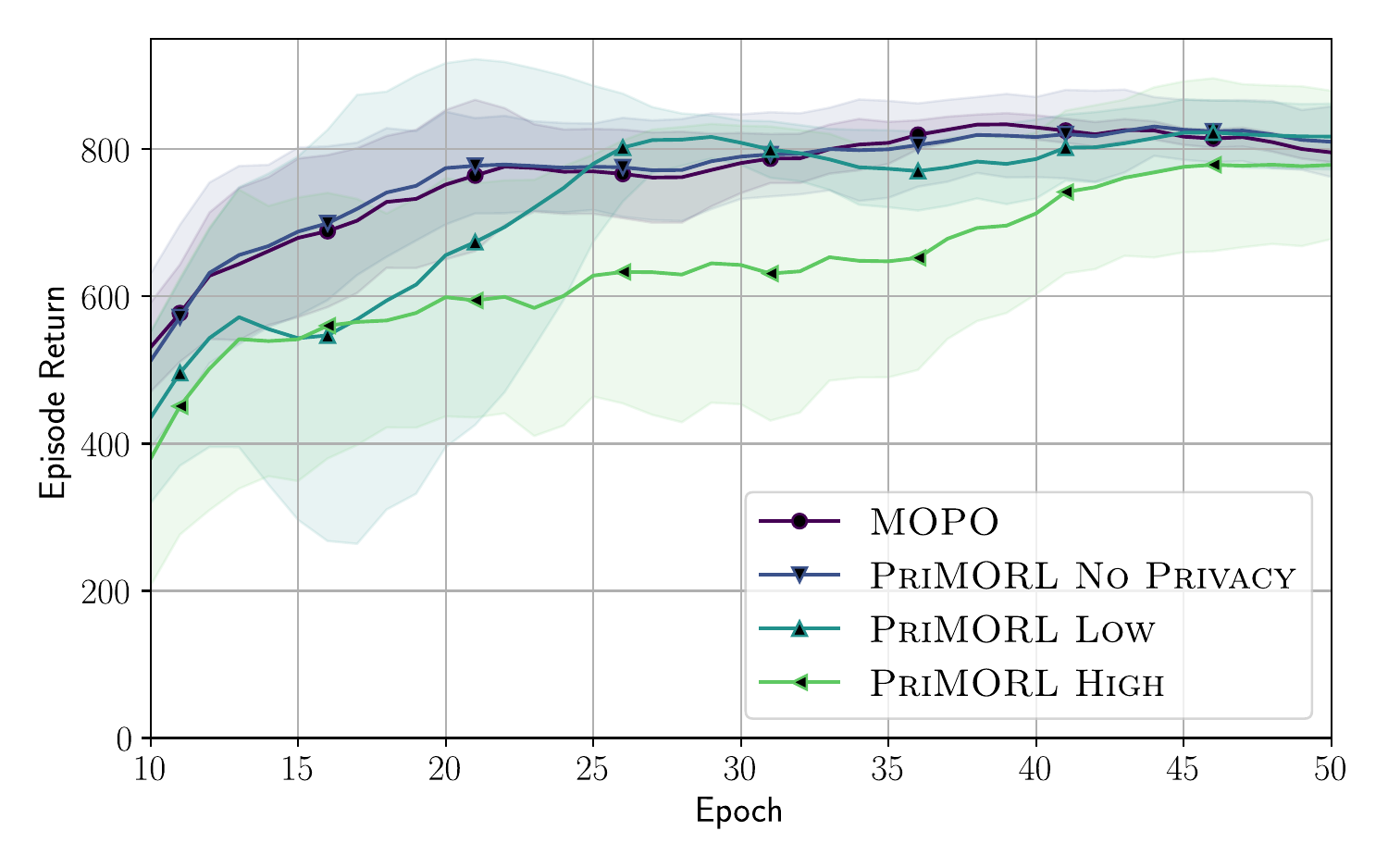}
\end{minipage}
\begin{minipage}{.32\textwidth}
%\vspace{0pt}
\includegraphics[width=\linewidth]{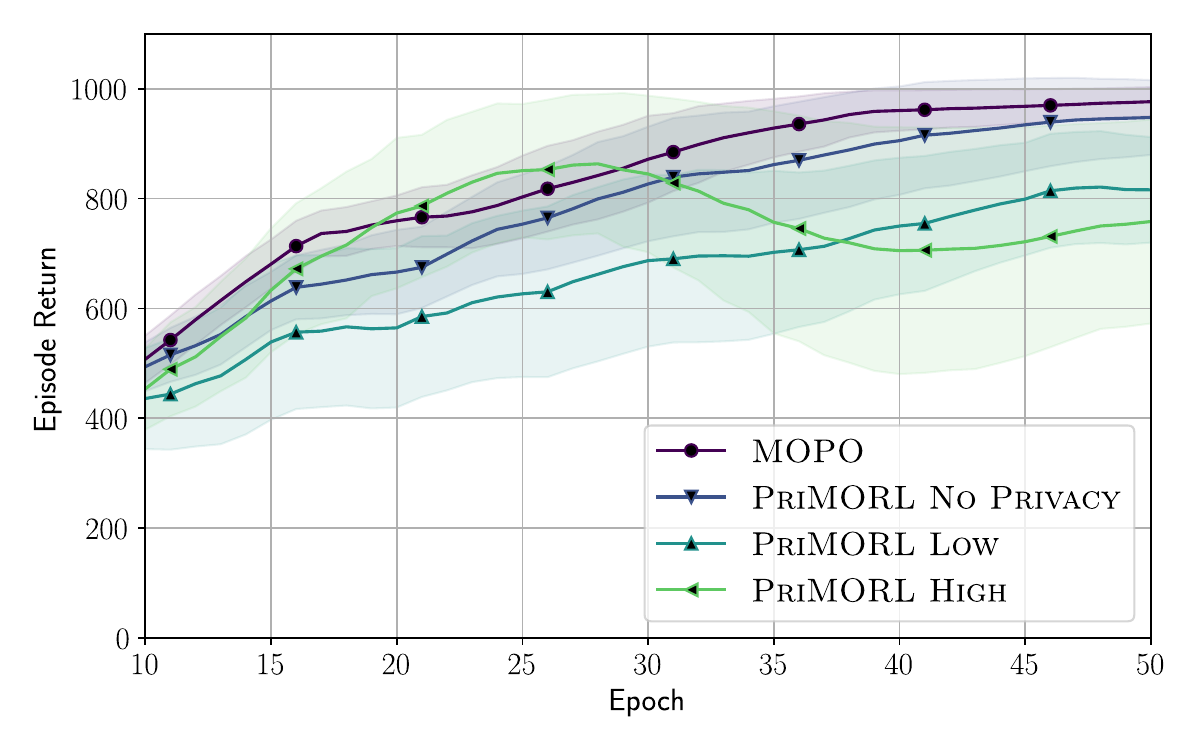}
\end{minipage}
%\hfill
\begin{minipage}{.32\textwidth}
%\vspace{0pt}
\includegraphics[width=\linewidth]{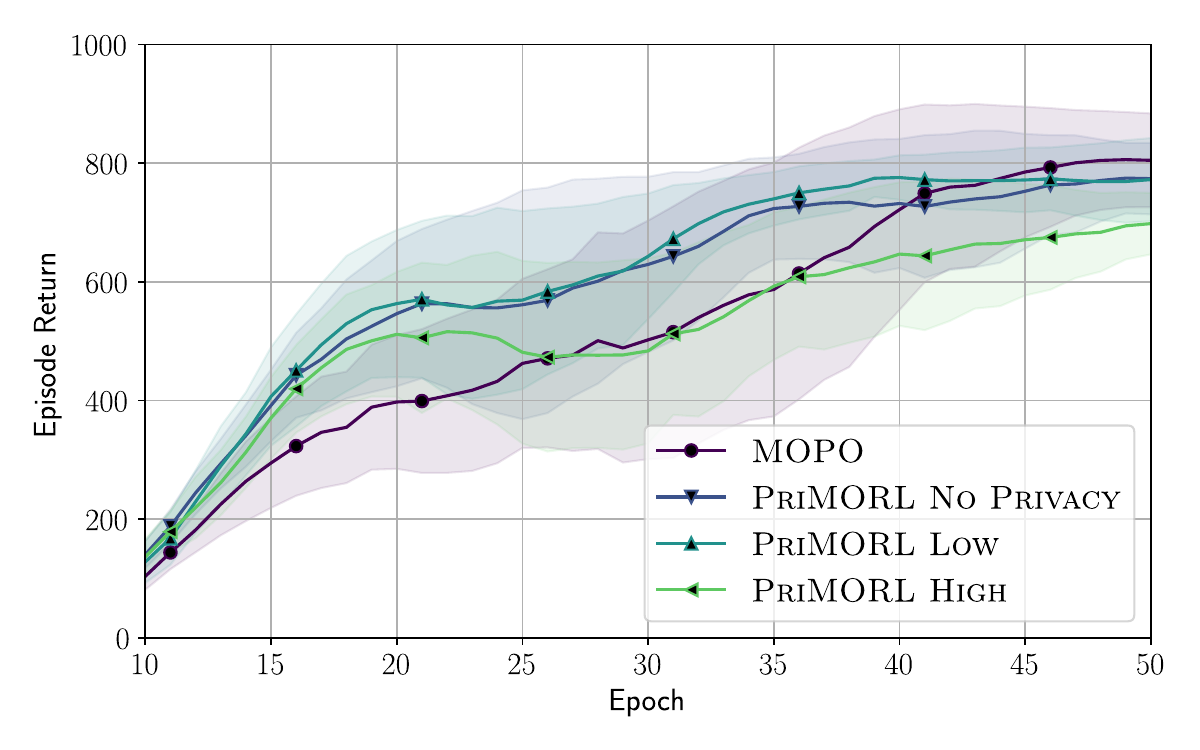}
\end{minipage}
\vspace{-2pt}
\caption{Learning curves on \textsc{Pendulum} (\textit{left}), \textsc{Balance} (\textit{middle}) and \textsc{Swingup} (\textit{right}).}
\label{fig:main_plots}
\end{figure*}

\subsection{Main Results}

%\subsubsection{Offline Dataset}

%It is now well-known that, to obtain acceptable privacy-utility trade-offs, differentially private training of machine learning models typically requires significantly more data compared to non-private training (see, \textit{e.g.}, \citet{ponomareva_how_2023}). As an example, \citet{McMahanRT018} consider datasets ranging in size from $10^6$ to $10^9$ users to train DP recurrent language models. In particular, such large datasets allow taking advantage of privacy amplification by subsampling, since $\epsilon$ scales with $q=\mathbb{E}\left[\vert \mathcal{U}_t \right] / K$, while at the same time reducing the magnitude of the noise added to the gradients, which scales in $1 / qK$. Consequently, we anticipate that standard offline RL benchmarks like RL Unplugged \citep{RL_unplugged} or D4RL \citep{fu2020d4rl}, whose datasets only contain dozens to thousands of trajectories, are not suitable for the learning of private RL agents (which was confirmed with preliminary experiments).

%For these reasons, we build our own offline dataset for \textsc{CartPole-SwingUp}, containing $K=30\,000$ trajectories (\textit{i.e.}, 30 million steps). 

% Bien préciser que les benchmarks non-private ne peuvent pas s'appliquer en privé en raison du nombre de données... Ce qui est aussi corroboré par la borne de DP-SGD (see above)

% Check methods of evaluation from RL Unplugged, D4RL

%\subsubsection{Results}

We present results on \textsc{Balance}, \textsc{SwingUp} and \textsc{Pendulum} for \textsc{PriMORL} and baselines in Table~\ref{tab:cartpole_results} and Figure~\ref{fig:main_plots}. Both report policy performance in the real MDP as the mean episodic return over 10 episodes per SAC training epoch. Average performance and 95\% confidence intervals are computed by re-training the model and the policy from scratch on at least 5 random seeds to assess the stability of the full training process.
Based on preliminary results, we use \textit{flat clipping} for both \textsc{Cartpole} tasks and \textit{per-layer clipping} for \textsc{Pendulum}. During policy optimization, \textsc{Swingup} uses $u_\text{MA}$ to estimate uncertainty while others use $u_\text{MPD}$. 
%Privacy budget $\epsilon$'s are computed for $\delta=10^{-5}$, \textit{i.e.} less than $K^{-1}$ as recommended in the literature. To tune the clipping norm, we set $z=0$ and progressively decreased $C$ until it started to adversely affect performance, and found that $C=10^{-2}$ provided the best results. Moreover, we set the sampling ratio to $q=10^{-3}$. We do not report results for \textsc{PriMORL No Clip} for this task as we found that the model optimized with \textsc{DP-FedAvg} diverges without clipping.
These results show a well-expected trade-off: performance tends to degrade with stronger privacy guarantees (\textit{i.e.}, smaller $\epsilon$'s), as the model training gets perturbed with higher levels of noise. Moreover, private model training makes the policy performance less stable over several runs, which is also expected since differential privacy adds another source of randomness during training. We notice that noise is not the sole factor that negatively impacts performance, as suggested by the gap between \textsc{MOPO} and the \textsc{PriMORL No Privacy}: gradient clipping and trajectory-level training, also contribute to performance degradation. In some cases, a small amount of DP noise might actually be beneficial, acting as a kind of regularization, as in \textsc{Swingup} and \textsc{Pendulum}. Moreover, experiments on \textsc{HalfCheetah} (Section~\ref{sec:exps_halfcheetah}) show that \textsc{PriMORL} performs worse in higher-dimensional tasks. This could be expected based on the theoretical analysis led in Section~\ref{sec:impact_privacy_policy_opt}, as DP training adds a dependence on the dimension $d$ of the task in the valuation gap.

Despite this trade-off, private agents trained with \textsc{PriMORL} remain competitive with \textsc{MOPO} for $\epsilon$ in the $10^1$ to $10^2$ range. For \textsc{Pendulum}, we plot policy performance against $\epsilon$ (Figure~\ref{fig:pendulum_f_eps} in appendix) and observe even no performance degradation until $\epsilon$ reaches the $1$ to $10$ range. While such $\epsilon$'s do not provide strong theoretical DP guarantees, they can provide adequate privacy protection in practical offline RL applications, as pointed out in \citet{ponomareva_how_2023} and backed by recent work on empirical privacy auditing (\textit{e.g.}, \cite{Carlini0EKS19, PonomarevaBV22}). In offline RL especially, the definition of DP assumes the adversary only has to discriminate between two precise neighboring datasets $D$ and $D^\prime = D \cup \{\tau\}$ as well as the release of all gradients and strong assumptions about the adversary, whereas in practice the adversary faces the much harder task of reconstructing a high-dimensional trajectory based on the output policy and limited side information only. We also point out that achieving a strong privacy-utility trade-off in offline RL requires access to datasets with a very large number of trajectories and that current benchmarks, with datasets of only dozens to thousands of trajectories, are insufficient for studying privacy effectively. In contrast, other fields often use datasets containing millions of users to ensure robust privacy guarantees, which would be very costly to study in offline RL. As an example, \citet{McMahanRT018} consider datasets with $10^6$ to $10^9$ users to train DP recurrent language models, and this is arguably the main reason why they achieve formal strong privacy guarantees. We discuss this in depth in appendix (Section~\ref{sec:app_price_privacy}) and provide both theoretical and practical evidence that increasing dataset size improves the privacy-performance trade-off for \textsc{PriMORL}, demonstrating even greater potential for our approach.

\begin{table}[t]
\tiny
\caption{Results for \textsc{Pendulum}, \textsc{Balance} and \textsc{Swingup}.}
\label{tab:cartpole_results}
\vskip 0.15in
\begin{center}
\begin{sc}
\begin{tabular}{l@{\hskip 15pt}cc@{\hskip 15pt}cc@{\hskip 15pt}cc}
\toprule
& \multicolumn{2}{c}{\textsc{Pendulum}} & \multicolumn{2}{c}{\textsc{CartPole-Balance}} & \multicolumn{2}{c}{\textsc{CartPole-Swingup}} \\
\midrule
Method & $\epsilon$ & Return & $\epsilon$ & Return & $\epsilon$ & Return \\
\midrule
MOPO & $\infty$ & $795.9 \pm 6.5$ & $\infty$ & $976.3 \pm 26.8$ & $\infty$ & $804.9 \pm 89.6$ \\
PriMORL No Priv.  & $\infty$ & $810.4 \pm 27.5 \textbf{ (101.8\%)}$ & $\infty$ & $947.5 \pm 68.3 \textbf{ (97.1\%)}$ & $\infty$ & $774.1 \pm 81.7 \textbf{ (96.17\%)}$ \\
\midrule
PriMORL Low  &  22.3 & $817.4 \pm 21.7 \textbf{ (102.7\%)}$ &  85.0 & $815.8 \pm 97.2 \textbf{ (83.6\%)}$ & 94.2 & $772.4 \pm 73.9 \textbf{ (95.96\%)}$ \\
PriMORL High  &  5.1 & $778.9 \pm 53.5 \textbf{ (97.9\%)}$ &  8.2 & $758.2 \pm 187.2 \textbf{ (77.7\%)}$ & 17.0 & $698.3 \pm 57.5 \textbf{ (86.75\%)}$ \\
\bottomrule
\end{tabular}
\end{sc}
\end{center}
\vskip -0.1in
\end{table}

\section{Discussion}
\label{sec:discussion}

While existing DP RL methods are limited to simple finite-horizon MDPs, we are the first to address deep offline RL with privacy guarantees in the infinite-horizon discounted setting, and propose a model-based approach named \textsc{PriMORL}. We go beyond the simple numerical simulation experiments that prevail in the literature and show empirically that \textsc{PriMORL} is capable of learning trajectory-level private policies in standard control tasks, with only limited performance cost. Although the reported privacy budgets are typically considered too large to stand as formal DP guarantees, we argue, based on recent studies on practical differential privacy (\textit{e.g.}, \citet{ponomareva_how_2023}), that they can offer satisfying privacy protection in practice, especially considering the worst-case nature of DP which can yield too pessimistic privacy budgets. Empirical evaluation of the robustness of our algorithm against privacy attacks, for which a rigorous benchmark has to be developed, will thus be an important research direction for future work. We further point out that our approach has the potential for achieving greater privacy-performance trade-offs given access to large enough offline datasets, hence calling for new benchmarks in the increasingly important field of private offline RL.
All in all, we believe that our work represents a significant step towards the deployment of private RL methods in more complex, high-dimensional control problems.

\bibliography{main}

\begin{thebibliography}{60}
\providecommand{\natexlab}[1]{#1}
\providecommand{\url}[1]{\texttt{#1}}
\expandafter\ifx\csname urlstyle\endcsname\relax
  \providecommand{\doi}[1]{doi: #1}\else
  \providecommand{\doi}{doi: \begingroup \urlstyle{rm}\Url}\fi

\bibitem[Abadi et~al.(2016)Abadi, Chu, Goodfellow, McMahan, Mironov, Talwar, and Zhang]{Abadi_2016}
Mart{\'{\i}}n Abadi, Andy Chu, Ian~J. Goodfellow, H.~Brendan McMahan, Ilya Mironov, Kunal Talwar, and Li~Zhang.
\newblock Deep learning with differential privacy.
\newblock In \emph{Proceedings of the 2016 {ACM} {SIGSAC} Conference on Computer and Communications Security}, pp.\  308--318. {ACM}, 2016.
\newblock URL \url{https://doi.org/10.1145/2976749.2978318}.

\bibitem[Argenson \& Dulac{-}Arnold(2021)Argenson and Dulac{-}Arnold]{ArgensonD21}
Arthur Argenson and Gabriel Dulac{-}Arnold.
\newblock Model-based offline planning.
\newblock In \emph{9th International Conference on Learning Representations, {ICLR}}, 2021.
\newblock URL \url{https://openreview.net/forum?id=OMNB1G5xzd4}.

\bibitem[Balle et~al.(2018)Balle, Barthe, and Gaboardi]{balle_privacy_2018}
Borja Balle, Gilles Barthe, and Marco Gaboardi.
\newblock Privacy amplification by subsampling: Tight analyses via couplings and divergences.
\newblock In \emph{Proceedings of NeurIPS}, 2018.
\newblock URL \url{https://proceedings.neurips.cc/paper/2018/hash/3b5020bb891119b9f5130f1fea9bd773-Abstract.html}.

\bibitem[Bassily et~al.(2014)Bassily, Smith, and Thakurta]{BassilyST14}
Raef Bassily, Adam~D. Smith, and Abhradeep Thakurta.
\newblock Private empirical risk minimization: Efficient algorithms and tight error bounds.
\newblock In \emph{55th {IEEE} Annual Symposium on Foundations of Computer Science, {FOCS} 2014, Philadelphia, PA, USA, October 18-21, 2014}, pp.\  464--473. {IEEE} Computer Society, 2014.
\newblock URL \url{https://doi.org/10.1109/FOCS.2014.56}.

\bibitem[Basu et~al.(2019)Basu, Dimitrakakis, and Tossou]{Basu2019}
Debabrota Basu, Christos Dimitrakakis, and Aristide C.~Y. Tossou.
\newblock Differential privacy for multi-armed bandits: What is it and what is its cost?
\newblock \emph{CoRR}, abs/1905.12298, 2019.
\newblock URL \url{http://arxiv.org/abs/1905.12298}.

\bibitem[Brockman et~al.(2016)Brockman, Cheung, Pettersson, Schneider, Schulman, Tang, and Zaremba]{openai_gym}
Greg Brockman, Vicki Cheung, Ludwig Pettersson, Jonas Schneider, John Schulman, Jie Tang, and Wojciech Zaremba.
\newblock {OpenAI} {Gym}.
\newblock \emph{CoRR}, abs/1606.01540, 2016.
\newblock URL \url{http://arxiv.org/abs/1606.01540}.

\bibitem[Bun \& Steinke(2016)Bun and Steinke]{BunCDP2016}
Mark Bun and Thomas Steinke.
\newblock Concentrated differential privacy: Simplifications, extensions, and lower bounds.
\newblock In \emph{Theory of Cryptography - 14th International Conference, {TCC} 2016-B, Beijing, China, October 31 - November 3, 2016, Proceedings, Part {I}}, volume 9985 of \emph{Lecture Notes in Computer Science}, pp.\  635--658, 2016.
\newblock URL \url{https://doi.org/10.1007/978-3-662-53641-4\_24}.

\bibitem[Carlini et~al.(2019)Carlini, Liu, Erlingsson, Kos, and Song]{Carlini0EKS19}
Nicholas Carlini, Chang Liu, {\'{U}}lfar Erlingsson, Jernej Kos, and Dawn Song.
\newblock The secret sharer: Evaluating and testing unintended memorization in neural networks.
\newblock In Nadia Heninger and Patrick Traynor (eds.), \emph{28th {USENIX} Security Symposium, {USENIX} Security 2019, Santa Clara, CA, USA, August 14-16, 2019}, pp.\  267--284. {USENIX} Association, 2019.
\newblock URL \url{https://www.usenix.org/conference/usenixsecurity19/presentation/carlini}.

\bibitem[Chowdhury \& Zhou(2021)Chowdhury and Zhou]{chowdhury_differentially_2021}
Sayak~Ray Chowdhury and Xingyu Zhou.
\newblock Differentially {Private} {Regret} {Minimization} in {Episodic} {Markov} {Decision} {Processes}, December 2021.
\newblock URL \url{http://arxiv.org/abs/2112.10599}.
\newblock arXiv:2112.10599 [cs, math].

\bibitem[Chua et~al.(2018)Chua, Calandra, McAllister, and Levine]{ChuaCML18}
Kurtland Chua, Roberto Calandra, Rowan McAllister, and Sergey Levine.
\newblock Deep reinforcement learning in a handful of trials using probabilistic dynamics models.
\newblock In \emph{Proceedings of NeurIPS}, 2018.
\newblock URL \url{https://proceedings.neurips.cc/paper/2018/hash/3de568f8597b94bda53149c7d7f5958c-Abstract.html}.

\bibitem[Cundy et~al.(2024)Cundy, Desai, and Ermon]{mutual_info_rl_2024}
Chris Cundy, Rishi Desai, and Stefano Ermon.
\newblock Privacy-constrained policies via mutual information regularized policy gradients.
\newblock In \emph{International Conference on Artificial Intelligence and Statistics}, volume 238 of \emph{Proceedings of Machine Learning Research}, pp.\  2809--2817. {PMLR}, 2024.
\newblock URL \url{https://proceedings.mlr.press/v238/j-cundy24a.html}.

\bibitem[den Hengst et~al.(2020)den Hengst, Grua, el~Hassouni, and Hoogendoorn]{Rl_perso_survey_2020}
Floris den Hengst, Eoin Grua, Ali el~Hassouni, and Mark Hoogendoorn.
\newblock Reinforcement learning for personalization: A systematic literature review.
\newblock \emph{Data Science}, 3:\penalty0 1--41, 04 2020.
\newblock \doi{10.3233/DS-200028}.

\bibitem[Dwork(2006)]{Dwork06}
Cynthia Dwork.
\newblock {Differential Privacy}.
\newblock In \emph{Proceedings of ICALP}, 2006.
\newblock URL \url{https://www.microsoft.com/en-us/research/publication/differential-privacy/}.

\bibitem[Dwork \& Rothblum(2016)Dwork and Rothblum]{DworkR16}
Cynthia Dwork and Guy~N. Rothblum.
\newblock Concentrated differential privacy.
\newblock \emph{CoRR}, abs/1603.01887, 2016.
\newblock URL \url{http://arxiv.org/abs/1603.01887}.

\bibitem[Dwork et~al.(2010)Dwork, Rothblum, and Vadhan]{DworkRV10}
Cynthia Dwork, Guy~N. Rothblum, and Salil~P. Vadhan.
\newblock Boosting and differential privacy.
\newblock In \emph{51th Annual {IEEE} Symposium on Foundations of Computer Science, {FOCS}}, pp.\  51--60. {IEEE} Computer Society, 2010.
\newblock URL \url{https://doi.org/10.1109/FOCS.2010.12}.

\bibitem[Fu et~al.(2020)Fu, Kumar, Nachum, Tucker, and Levine]{fu2020d4rl}
Justin Fu, Aviral Kumar, Ofir Nachum, George Tucker, and Sergey Levine.
\newblock {D4RL:} datasets for deep data-driven reinforcement learning.
\newblock \emph{CoRR}, abs/2004.07219, 2020.
\newblock URL \url{https://arxiv.org/abs/2004.07219}.

\bibitem[Fujimoto et~al.(2019)Fujimoto, Meger, and Precup]{FujimotoMP19}
Scott Fujimoto, David Meger, and Doina Precup.
\newblock Off-policy deep reinforcement learning without exploration.
\newblock In \emph{Proceedings of the 36th International Conference on Machine Learning, {ICML} 2019}, volume~97 of \emph{Proceedings of Machine Learning Research}, pp.\  2052--2062. {PMLR}, 2019.
\newblock URL \url{http://proceedings.mlr.press/v97/fujimoto19a.html}.

\bibitem[Garcelon et~al.(2021)Garcelon, Perchet, Pike{-}Burke, and Pirotta]{GarcelonPPP21}
Evrard Garcelon, Vianney Perchet, Ciara Pike{-}Burke, and Matteo Pirotta.
\newblock Local differential privacy for regret minimization in reinforcement learning.
\newblock In \emph{Proceedings of NeurIPS}, 2021.
\newblock URL \url{https://proceedings.neurips.cc/paper/2021/hash/580760fb5def6e2ca8eaf601236d5b08-Abstract.html}.

\bibitem[Gomrokchi et~al.(2023)Gomrokchi, Amin, Aboutalebi, Wong, and Precup]{gomrokchi_membership_2022}
Maziar Gomrokchi, Susan Amin, Hossein Aboutalebi, Alexander Wong, and Doina Precup.
\newblock Membership inference attacks against temporally correlated data in deep reinforcement learning.
\newblock \emph{{IEEE} Access}, 11:\penalty0 42796--42808, 2023.
\newblock URL \url{https://doi.org/10.1109/ACCESS.2023.3270860}.

\bibitem[G{\"{u}}l{\c{c}}ehre et~al.(2020)G{\"{u}}l{\c{c}}ehre, Wang, Novikov, Paine, Colmenarejo, Zolna, Agarwal, Merel, Mankowitz, Paduraru, Dulac{-}Arnold, Li, Norouzi, Hoffman, Heess, and de~Freitas]{RL_unplugged}
{\c{C}}aglar G{\"{u}}l{\c{c}}ehre, Ziyu Wang, Alexander Novikov, Thomas Paine, Sergio~G{\'{o}}mez Colmenarejo, Konrad Zolna, Rishabh Agarwal, Josh Merel, Daniel~J. Mankowitz, Cosmin Paduraru, Gabriel Dulac{-}Arnold, Jerry Li, Mohammad Norouzi, Matthew Hoffman, Nicolas Heess, and Nando de~Freitas.
\newblock {RL} unplugged: {A} collection of benchmarks for offline reinforcement learning.
\newblock In \emph{Proceedings of NeurIPS}, 2020.
\newblock URL \url{https://proceedings.neurips.cc/paper/2020/hash/51200d29d1fc15f5a71c1dab4bb54f7c-Abstract.html}.

\bibitem[Haarnoja et~al.(2018)Haarnoja, Zhou, Abbeel, and Levine]{Haarnoja_SAC_2018}
Tuomas Haarnoja, Aurick Zhou, Pieter Abbeel, and Sergey Levine.
\newblock Soft actor-critic: Off-policy maximum entropy deep reinforcement learning with a stochastic actor.
\newblock In \emph{Proceedings of the 35th International Conference on Machine Learning, {ICML} 2018}, volume~80 of \emph{Proceedings of Machine Learning Research}, pp.\  1856--1865. {PMLR}, 2018.
\newblock URL \url{http://proceedings.mlr.press/v80/haarnoja18b.html}.

\bibitem[Kang et~al.(2023)Kang, Li, Liu, and Wang]{KangLLW23}
Yilin Kang, Jian Li, Yong Liu, and Weiping Wang.
\newblock Data heterogeneity differential privacy: From theory to algorithm.
\newblock In \emph{Computational Science - {ICCS} 2023 - 23rd International Conference, Prague, Czech Republic, July 3-5, 2023, Proceedings, Part {I}}, volume 14073 of \emph{Lecture Notes in Computer Science}, pp.\  119--133. Springer, 2023.
\newblock URL \url{https://doi.org/10.1007/978-3-031-35995-8\_9}.

\bibitem[Kearns \& Singh(2002)Kearns and Singh]{Kearns2002}
Michael~J. Kearns and Satinder Singh.
\newblock Near-optimal reinforcement learning in polynomial time.
\newblock \emph{Mach. Learn.}, 49\penalty0 (2-3):\penalty0 209--232, 2002.
\newblock URL \url{https://doi.org/10.1023/A:1017984413808}.

\bibitem[Kidambi et~al.(2020)Kidambi, Rajeswaran, Netrapalli, and Joachims]{kidambi_morel_2021}
Rahul Kidambi, Aravind Rajeswaran, Praneeth Netrapalli, and Thorsten Joachims.
\newblock {MOReL}: Model-based offline reinforcement learning.
\newblock In \emph{Proceedings of NeurIPS}, 2020.
\newblock URL \url{https://proceedings.neurips.cc/paper/2020/hash/f7efa4f864ae9b88d43527f4b14f750f-Abstract.html}.

\bibitem[Kim \& Oh(2023)Kim and Oh]{kim_count_2023}
Byeongchan Kim and Min~Hwan Oh.
\newblock Model-based offline reinforcement learning with count-based conservatism.
\newblock In \emph{International Conference on Machine Learning, {ICML} 2023}, volume 202 of \emph{Proceedings of Machine Learning Research}, pp.\  16728--16746. {PMLR}, 2023.
\newblock URL \url{https://proceedings.mlr.press/v202/kim23q.html}.

\bibitem[Kiran et~al.(2022)Kiran, Sobh, Talpaert, Mannion, Sallab, Yogamani, and P{\'{e}}rez]{autonomous_driving_survey}
B.~Ravi Kiran, Ibrahim Sobh, Victor Talpaert, Patrick Mannion, Ahmad A.~Al Sallab, Senthil~Kumar Yogamani, and Patrick P{\'{e}}rez.
\newblock Deep reinforcement learning for autonomous driving: {A} survey.
\newblock \emph{{IEEE} Trans. Intell. Transp. Syst.}, 23\penalty0 (6):\penalty0 4909--4926, 2022.
\newblock URL \url{https://doi.org/10.1109/TITS.2021.3054625}.

\bibitem[Levine et~al.(2020)Levine, Kumar, Tucker, and Fu]{Levine_Offline_RL_2020}
Sergey Levine, Aviral Kumar, George Tucker, and Justin Fu.
\newblock Offline reinforcement learning: Tutorial, review, and perspectives on open problems.
\newblock \emph{CoRR}, abs/2005.01643, 2020.
\newblock URL \url{https://arxiv.org/abs/2005.01643}.

\bibitem[Liao et~al.(2021)Liao, He, and Gu]{Liao2021}
Chonghua Liao, Jiafan He, and Quanquan Gu.
\newblock Locally differentially private reinforcement learning for linear mixture markov decision processes.
\newblock \emph{CoRR}, abs/2110.10133, 2021.
\newblock URL \url{https://arxiv.org/abs/2110.10133}.

\bibitem[Lillicrap et~al.(2016)Lillicrap, Hunt, Pritzel, Heess, Erez, Tassa, Silver, and Wierstra]{DDPG}
Timothy~P. Lillicrap, Jonathan~J. Hunt, Alexander Pritzel, Nicolas Heess, Tom Erez, Yuval Tassa, David Silver, and Daan Wierstra.
\newblock Continuous control with deep reinforcement learning.
\newblock In \emph{4th International Conference on Learning Representations, {ICLR} 2016}, 2016.
\newblock URL \url{http://arxiv.org/abs/1509.02971}.

\bibitem[Liu et~al.(2022)Liu, Shen, and Pan]{rl_treatment_recommendation}
Mingyang Liu, Xiaotong Shen, and Wei Pan.
\newblock Deep reinforcement learning for personalized treatment recommendation.
\newblock \emph{Statistics in Medicine}, 41, 06 2022.

\bibitem[Liu et~al.(2020)Liu, See, Ngiam, Celi, Sun, and Feng]{SiqiHealthcare2020}
Siqi Liu, Kay~Choong See, Kee~Yuan Ngiam, Leo~Anthony Celi, Xingzhi Sun, and Mengling Feng.
\newblock Reinforcement learning for clinical decision support in critical care: Comprehensive review.
\newblock \emph{J Med Internet Res}, 22\penalty0 (7):\penalty0 e18477, Jul 2020.
\newblock URL \url{https://www.jmir.org/2020/7/e18477}.

\bibitem[Lu et~al.(2022)Lu, Ball, Parker{-}Holder, Osborne, and Roberts]{mbrl_uncertainty_2022}
Cong Lu, Philip~J. Ball, Jack Parker{-}Holder, Michael~A. Osborne, and Stephen~J. Roberts.
\newblock Revisiting design choices in offline model based reinforcement learning.
\newblock In \emph{The Tenth International Conference on Learning Representations, {ICLR} 2022}, 2022.
\newblock URL \url{https://openreview.net/forum?id=zz9hXVhf40}.

\bibitem[Luyo et~al.(2021)Luyo, Garcelon, Lazaric, and Pirotta]{luyo2021differentially}
Paul Luyo, Evrard Garcelon, Alessandro Lazaric, and Matteo Pirotta.
\newblock Differentially private exploration in reinforcement learning with linear representation, 2021.
\newblock URL \url{https://arxiv.org/abs/2112.01585}.

\bibitem[McMahan et~al.(2017)McMahan, Moore, Ramage, Hampson, and y~Arcas]{McMahan_FedAvg_17}
Brendan McMahan, Eider Moore, Daniel Ramage, Seth Hampson, and Blaise~Ag{\"{u}}era y~Arcas.
\newblock Communication-efficient learning of deep networks from decentralized data.
\newblock In \emph{Proceedings of the 20th International Conference on Artificial Intelligence and Statistics, {AISTATS} 2017}, volume~54 of \emph{Proceedings of Machine Learning Research}, pp.\  1273--1282. {PMLR}, 2017.
\newblock URL \url{http://proceedings.mlr.press/v54/mcmahan17a.html}.

\bibitem[McMahan et~al.(2018)McMahan, Ramage, Talwar, and Zhang]{McMahanRT018}
H.~Brendan McMahan, Daniel Ramage, Kunal Talwar, and Li~Zhang.
\newblock Learning differentially private recurrent language models.
\newblock In \emph{6th International Conference on Learning Representations, {ICLR}}, 2018.
\newblock URL \url{https://openreview.net/forum?id=BJ0hF1Z0b}.

\bibitem[Mironov(2017)]{MironovRDP2017}
Ilya Mironov.
\newblock R{\'{e}}nyi differential privacy.
\newblock In \emph{2017 {IEEE} 30th Computer Security Foundations Symposium ({CSF})}. {IEEE}, aug 2017.
\newblock URL \url{https://doi.org/10.1109%2Fcsf.2017.11}.

\bibitem[Moerland et~al.(2023)Moerland, Broekens, Plaat, and Jonker]{MoerlandMBRLSurvey2023}
Thomas~M. Moerland, Joost Broekens, Aske Plaat, and Catholijn~M. Jonker.
\newblock Model-based reinforcement learning: {A} survey.
\newblock \emph{Found. Trends Mach. Learn.}, 16\penalty0 (1):\penalty0 1--118, 2023.
\newblock URL \url{https://doi.org/10.1561/2200000086}.

\bibitem[Ngo et~al.(2022)Ngo, Vietri, and Wu]{ngo_linear_2022}
Dung Daniel~T. Ngo, Giuseppe Vietri, and Steven Wu.
\newblock Improved regret for differentially private exploration in linear {MDP}.
\newblock In Kamalika Chaudhuri, Stefanie Jegelka, Le~Song, Csaba Szepesv{\'{a}}ri, Gang Niu, and Sivan Sabato (eds.), \emph{International Conference on Machine Learning, {ICML} 2022, 17-23 July 2022, Baltimore, Maryland, {USA}}, volume 162 of \emph{Proceedings of Machine Learning Research}, pp.\  16529--16552. {PMLR}, 2022.
\newblock URL \url{https://proceedings.mlr.press/v162/ngo22a.html}.

\bibitem[Pan et~al.(2019)Pan, Wang, Zhang, Li, Yi, and Song]{pan_how_2019}
Xinlei Pan, Weiyao Wang, Xiaoshuai Zhang, Bo~Li, Jinfeng Yi, and Dawn Song.
\newblock How {You} {Act} {Tells} a {Lot}: {Privacy}-{Leaking} {Attack} on {Deep} {Reinforcement} {Learning}.
\newblock \emph{Reinforcement Learning}, 2019.

\bibitem[Ponomareva et~al.(2022)Ponomareva, Bastings, and Vassilvitskii]{PonomarevaBV22}
Natalia Ponomareva, Jasmijn Bastings, and Sergei Vassilvitskii.
\newblock Training text-to-text transformers with privacy guarantees.
\newblock In Smaranda Muresan, Preslav Nakov, and Aline Villavicencio (eds.), \emph{Findings of the Association for Computational Linguistics: {ACL} 2022, Dublin, Ireland, May 22-27, 2022}, pp.\  2182--2193. Association for Computational Linguistics, 2022.
\newblock \doi{10.18653/v1/2022.findings-acl.171}.
\newblock URL \url{https://doi.org/10.18653/v1/2022.findings-acl.171}.

\bibitem[Ponomareva et~al.(2023)Ponomareva, Hazimeh, Kurakin, Xu, Denison, McMahan, Vassilvitskii, Chien, and Thakurta]{ponomareva_how_2023}
Natalia Ponomareva, Hussein Hazimeh, Alex Kurakin, Zheng Xu, Carson Denison, H.~Brendan McMahan, Sergei Vassilvitskii, Steve Chien, and Abhradeep Thakurta.
\newblock How to {DP}-fy {ML}: {A} {Practical} {Guide} to {Machine} {Learning} with {Differential} {Privacy}, March 2023.
\newblock URL \url{http://arxiv.org/abs/2303.00654}.
\newblock arXiv:2303.00654 [cs, stat].

\bibitem[Prakash et~al.(2022)Prakash, Husain, Paruchuri, and Gujar]{prakash_how_2022}
Kritika Prakash, Fiza Husain, Praveen Paruchuri, and Sujit Gujar.
\newblock How {Private} {Is} {Your} {RL} {Policy}? {An} {Inverse} {RL} {Based} {Analysis} {Framework}.
\newblock \emph{Proceedings of the AAAI Conference on Artificial Intelligence}, 36\penalty0 (7):\penalty0 8009--8016, June 2022.
\newblock ISSN 2374-3468, 2159-5399.
\newblock \doi{10.1609/aaai.v36i7.20772}.
\newblock URL \url{https://ojs.aaai.org/index.php/AAAI/article/view/20772}.

\bibitem[Prudencio et~al.(2022)Prudencio, M{\'{a}}ximo, and Colombini]{Figueirido2022}
Rafael~Figueiredo Prudencio, Marcos R. O.~A. M{\'{a}}ximo, and Esther~Luna Colombini.
\newblock A survey on offline reinforcement learning: Taxonomy, review, and open problems.
\newblock \emph{CoRR}, abs/2203.01387, 2022.
\newblock URL \url{https://doi.org/10.48550/arXiv.2203.01387}.

\bibitem[Qiao \& Wang(2023{\natexlab{a}})Qiao and Wang]{qiao_offline_2022}
Dan Qiao and Yu{-}Xiang Wang.
\newblock Offline reinforcement learning with differential privacy.
\newblock In \emph{Proceedings of NeurIPS}, 2023{\natexlab{a}}.
\newblock URL \url{http://papers.nips.cc/paper\_files/paper/2023/hash/c1aaf7c3f306fe94f77236dc0756d771-Abstract-Conference.html}.

\bibitem[Qiao \& Wang(2023{\natexlab{b}})Qiao and Wang]{qiao_tabular_2023}
Dan Qiao and Yu{-}Xiang Wang.
\newblock Near-optimal differentially private reinforcement learning.
\newblock In Francisco J.~R. Ruiz, Jennifer~G. Dy, and Jan{-}Willem van~de Meent (eds.), \emph{International Conference on Artificial Intelligence and Statistics, 25-27 April 2023, Palau de Congressos, Valencia, Spain}, volume 206 of \emph{Proceedings of Machine Learning Research}, pp.\  9914--9940. {PMLR}, 2023{\natexlab{b}}.
\newblock URL \url{https://proceedings.mlr.press/v206/qiao23a.html}.

\bibitem[Rigaki \& Garcia(2020)Rigaki and Garcia]{RigakiAttacksSurvey2020}
Maria Rigaki and Sebastian Garcia.
\newblock A survey of privacy attacks in machine learning.
\newblock \emph{CoRR}, abs/2007.07646, 2020.
\newblock URL \url{https://arxiv.org/abs/2007.07646}.

\bibitem[Shalev{-}Shwartz et~al.(2009)Shalev{-}Shwartz, Shamir, Srebro, and Sridharan]{ShalevSchwartz2009}
Shai Shalev{-}Shwartz, Ohad Shamir, Nathan Srebro, and Karthik Sridharan.
\newblock Stochastic convex optimization.
\newblock In \emph{{COLT} 2009}, 2009.
\newblock URL \url{http://www.cs.mcgill.ca/\%7Ecolt2009/papers/018.pdf\#page=1}.

\bibitem[Shokri et~al.(2017)Shokri, Stronati, Song, and Shmatikov]{Shokri2017}
R.~Shokri, M.~Stronati, C.~Song, and V.~Shmatikov.
\newblock Membership inference attacks against machine learning models.
\newblock In \emph{2017 IEEE Symposium on Security and Privacy (SP)}, pp.\  3--18. IEEE Computer Society, 2017.
\newblock URL \url{https://doi.ieeecomputersociety.org/10.1109/SP.2017.41}.

\bibitem[Singh et~al.(2022)Singh, Kumar, and Singh]{SinghKSRobotic22}
Bharat Singh, Rajesh Kumar, and Vinay~Pratap Singh.
\newblock Reinforcement learning in robotic applications: a comprehensive survey.
\newblock \emph{Artif. Intell. Rev.}, 55\penalty0 (2):\penalty0 945--990, 2022.
\newblock URL \url{https://doi.org/10.1007/s10462-021-09997-9}.

\bibitem[Sutton \& Barto(1998)Sutton and Barto]{SuttonB98}
Richard~S. Sutton and Andrew~G. Barto.
\newblock \emph{Reinforcement learning - an introduction}.
\newblock Adaptive computation and machine learning. {MIT} Press, 1998.
\newblock ISBN 978-0-262-19398-6.
\newblock URL \url{https://www.worldcat.org/oclc/37293240}.

\bibitem[Tassa et~al.(2018)Tassa, Doron, Muldal, Erez, Li, de~Las~Casas, Budden, Abdolmaleki, Merel, Lefrancq, Lillicrap, and Riedmiller]{DM_Control_Suite}
Yuval Tassa, Yotam Doron, Alistair Muldal, Tom Erez, Yazhe Li, Diego de~Las~Casas, David Budden, Abbas Abdolmaleki, Josh Merel, Andrew Lefrancq, Timothy~P. Lillicrap, and Martin~A. Riedmiller.
\newblock Deepmind control suite.
\newblock \emph{CoRR}, abs/1801.00690, 2018.
\newblock URL \url{http://arxiv.org/abs/1801.00690}.

\bibitem[Todorov et~al.(2012)Todorov, Erez, and Tassa]{mujoco}
Emanuel Todorov, Tom Erez, and Yuval Tassa.
\newblock Mujoco: {A} physics engine for model-based control.
\newblock In \emph{2012 {IEEE/RSJ} International Conference on Intelligent Robots and Systems, {IROS} 2012, Vilamoura, Algarve, Portugal, October 7-12, 2012}, pp.\  5026--5033. {IEEE}, 2012.
\newblock URL \url{https://doi.org/10.1109/IROS.2012.6386109}.

\bibitem[Tossou \& Dimitrakakis(2016)Tossou and Dimitrakakis]{TossouD16}
Aristide Tossou and Christos Dimitrakakis.
\newblock {Algorithms for Differentially Private Multi-Armed Bandits}.
\newblock In \emph{Proceedings of AAAI}, 2016.
\newblock URL \url{https://aaai.org/papers/212-algorithms-for-differentially-private-multi-armed-bandits/}.

\bibitem[Vietri et~al.(2020)Vietri, Balle, Krishnamurthy, and Wu]{vietri2020private}
Giuseppe Vietri, Borja Balle, Akshay Krishnamurthy, and Zhiwei~Steven Wu.
\newblock Private reinforcement learning with {PAC} and regret guarantees.
\newblock In \emph{Proceedings of the 37th International Conference on Machine Learning, {ICML} 2020, 13-18 July 2020, Virtual Event}, volume 119 of \emph{Proceedings of Machine Learning Research}, pp.\  9754--9764. {PMLR}, 2020.
\newblock URL \url{http://proceedings.mlr.press/v119/vietri20a.html}.

\bibitem[Wang \& Hegde(2019)Wang and Hegde]{Wang2019}
Baoxiang Wang and Nidhi Hegde.
\newblock Privacy-preserving q-learning with functional noise in continuous spaces.
\newblock In \emph{Advances in Neural Information Processing Systems}, volume~32, 2019.
\newblock URL \url{https://proceedings.neurips.cc/paper_files/paper/2019/file/6646b06b90bd13dabc11ddba01270d23-Paper.pdf}.

\bibitem[Wawrzynski(2009)]{halfcheetah}
Pawel Wawrzynski.
\newblock A cat-like robot real-time learning to run.
\newblock In \emph{Adaptive and Natural Computing Algorithms, 9th International Conference, {ICANNGA} 2009, Kuopio, Finland, April 23-25, 2009, Revised Selected Papers}, volume 5495 of \emph{Lecture Notes in Computer Science}, pp.\  380--390. Springer, 2009.
\newblock URL \url{https://doi.org/10.1007/978-3-642-04921-7\_39}.

\bibitem[Xu et~al.(2020)Xu, Li, and Yu]{Xu2020}
Tian Xu, Ziniu Li, and Yang Yu.
\newblock Error bounds of imitating policies and environments.
\newblock In \emph{Advances in Neural Information Processing Systems 33: Annual Conference on Neural Information Processing Systems 2020, NeurIPS 2020, December 6-12, 2020, virtual}, 2020.
\newblock URL \url{https://proceedings.neurips.cc/paper/2020/hash/b5c01503041b70d41d80e3dbe31bbd8c-Abstract.html}.

\bibitem[Yu et~al.(2020)Yu, Thomas, Yu, Ermon, Zou, Levine, Finn, and Ma]{yu_mopo_2020}
Tianhe Yu, Garrett Thomas, Lantao Yu, Stefano Ermon, James~Y. Zou, Sergey Levine, Chelsea Finn, and Tengyu Ma.
\newblock {MOPO:} model-based offline policy optimization.
\newblock In \emph{Proceedings of NeurIPS}, 2020.
\newblock URL \url{https://proceedings.neurips.cc/paper/2020/hash/a322852ce0df73e204b7e67cbbef0d0a-Abstract.html}.

\bibitem[Zheng et~al.(2018)Zheng, Zhang, Zheng, Xiang, Yuan, Xie, and Li]{news_recommendation}
Guanjie Zheng, Fuzheng Zhang, Zihan Zheng, Yang Xiang, Nicholas~Jing Yuan, Xing Xie, and Zhenhui Li.
\newblock {DRN:} {A} deep reinforcement learning framework for news recommendation.
\newblock In \emph{Proceedings of the 2018 World Wide Web Conference on World Wide Web, {WWW} 2018, Lyon, France, April 23-27, 2018}, pp.\  167--176. {ACM}, 2018.
\newblock URL \url{https://doi.org/10.1145/3178876.3185994}.

\bibitem[Zhou(2022)]{zhou_differentially_2022}
Xingyu Zhou.
\newblock Differentially {Private} {Reinforcement} {Learning} with {Linear} {Function} {Approximation}.
\newblock \emph{Proceedings of the ACM on Measurement and Analysis of Computing Systems}, 6\penalty0 (1):\penalty0 1--27, 2022.
\newblock URL \url{https://dl.acm.org/doi/10.1145/3508028}.

\end{thebibliography}
\bibliographystyle{iclr2025_conference}

%%%%%%%%%%%%%%%%%%%%%%%%%%%%%%%%%%%%%%%%%%%%%%%%%%%%%%%%%%%%

\newpage
\appendix

\section{Proofs} \label{sec:proofs}

\DPModelPrivacy*

\begin{proof}
    Theorem 1 from \citet {McMahanRT018} shows that the moments accounting method from \citet{Abadi_2016} computes correctly the privacy loss of \textsc{DP-FedAvg} at user-level for the noise multiplier $z = \sigma / \mathbb{C}$ with $\mathbb{C} = C / qK$ if, for each user $u_k$, the clipped gradient $\Delta^{\text{clipped}}_k$ computed from $u_k$'s data has sensitivity bounded by $C$ (referred to as \textbf{Condition 1}). With \textsc{TDP Model Ensemble Training}, we train the model ensemble as a single big model: at each training iteration, the same input batch is processed forward by all models in a single pass, a single loss is computed for the ensemble, and the parameters are then updated in a single backward pass. The ensemble of models can therefore be seen as a concatenation of all individual models, equivalent to a larger model $\theta = \left(\theta_i\right)_{i=1}^N$. We can therefore extend this theorem by mapping users in federating learning to trajectories in offline RL, as long as \textbf{Condition 1} holds for every trajectory $\tau_k$.

    Since we use \textbf{ensemble clipping}, we verify that, for trajectory $\tau_k$, the ensemble gradient $\Delta^{\textsc{clipped}}_k = \left(\Delta^{\textsc{clipped}}_{i, k}\right)$ has sensitivity bounded by $C$. With \textbf{flat ensemble clipping}, the gradient of each model $i \in [\![1, N]\!]$ is clipped by a factor $C_i = \frac{C}{\sqrt{N}}$ (see Algorithm~\ref{alg:per_layer_clipping}). By construction, $\Delta^{\textsc{clipped}}_{i, k}$ has sensitivity bounded by $C_i$, \textit{i.e.}, $\max_{d(D, D^\prime) = 1} \Vert \Delta^{\textsc{clipped}}_{i, k}(D) - \Delta^{\textsc{clipped}}_{i, k}(D^\prime) \Vert_2 \le C_i$. Therefore, for two neighboring datasets $D$ and $D^\prime$:
    \begin{align*}
        \Vert \Delta^{\textsc{clipped}}_{k}(D) - \Delta^{\textsc{clipped}}_{k}(D^\prime) \Vert_2 &= \Vert \left(\Delta^{\textsc{clipped}}_{k}(D) - \Delta^{\textsc{clipped}}_{k}(D^\prime)\right)_{i=1}^N \Vert_2 \\
        &= \sqrt{\sum_{i=1}^N \Vert \Delta^{\textsc{clipped}}_{i, k}(D) - \Delta^{\textsc{clipped}}_{i, k}(D^\prime) \Vert_2^2} \\
        &\le \sqrt{\sum_{i=1}^N C_i^2} \\
        &= \sqrt{\sum_{i=1}^N \frac{C^2}{N}} \\
        &= C \enspace.
    \end{align*}
    This implies $\max_{d(D, D^\prime) = 1}  \Vert \Delta^{\textsc{clipped}}_{k}(D) - \Delta^{\textsc{clipped}}_{k}(D^\prime) \Vert_2 \le C$:  $\Delta^{\textsc{clipped}}_{k}$ has sensitivity bounded by $C$. We can derive the same proof for \textbf{per-layer ensemble clipping}. Therefore, Theorem 1 from \citet{McMahanRT018} holds for \textsc{TDP Model Ensemble Training}.
    
    We can therefore use the moments accountant $\epsilon^\text{MA}$ to compute, given $z >0$, $\delta \in (0,1)$, $q \in (0,1)$ and $T \in \mathbb{N}$, the total privacy budget $\epsilon$ spent by Algorithm \ref{alg:dp_model_training}, \textit{i.e.}, $\epsilon = \epsilon^{\text{MA}}(z, q, T, \delta)$.

    The dyanmics model output by Algorithm~\ref{alg:dp_model_training} is therefore $(\epsilon, \delta)$-TDP.
\end{proof}

\DPPolicyPrivacy*

\begin{proof}
    First, we establish that the pessimistic MDP $\Tilde{M}$ is private for $u \in \{u_\text{MA}, u_\text{MPD}\}$. By Theorem~\ref{thm:dp_model_privacy}, both the mean estimators $\{\mu_{\phi_i}\}_{i=1}^N$ the covariance estimators $\{\Sigma_{\psi_i}\}_{i=1}^N$ are private. Therefore, both uncertainty estimators $u_\text{MA}(s,a) = \lVert \Sigma_{\psi_i}(s,a) \rVert_F$ and $u_\text{MPD}(s,a) = \max_{i, j} \Vert f_{\phi_i} - f_{\phi_j} \Vert_2$, as data-independent transformations of the above quantities, are also private thanks to the post-processing property of DP. Therefore, the pessimistic model $\Tilde{M}$ remains $(\epsilon, \delta)$-TDP.
    
    Now, we can think of SAC model-based policy optimization (Algorithm~\ref{alg:sac_training}) as an abstract, randomized function $h_\Pi$, that takes as input $\hat{M}$ and outputs as policy $\hat{\pi}$. Furthermore, let $h_M$ denote the mechanism that takes as input the private offline dataset $\mathcal{D}_K$ and outputs the private pessimistic model $\hat{M}$, and which is $(\epsilon, \delta)$-TDP following \ref{thm:dp_model_privacy}. We observe that $h = h_\Pi \circ h_M$, where $h$ is the global offline RL algorithm which is the object of Definition \ref{def:tdp}. Since SAC only uses data from the model, as stated in Section \ref{subsubsec:policy_opt}, $h_\Pi$ is independent of the private offline data $\mathcal{D}_K$. In other words, $h_\Pi$ is a data-independent transformation of the private mechanism $h_M$. Thanks again to the post-processing property of differential privacy, $h$ is also $(\epsilon, \delta)$-TDP.
\end{proof}

We now prove the following two propositions:
\PerfPolicyNonPrivate*
\PerfPolicy*

\begin{proof}
    Let $\mathcal{F}$ denote the function class of the model. The model is estimated by maximizing the likelihood of the data $\mathcal{D}_K = \left(s_i, a_i, s^\prime_i\right)_{i=1}^N$, which is collected by an unknown behavioral policy $\pi^B$. This is equivalent to minimizing the negative log-likelihood. The population risk of the estimated model $\hat{P}$ obtained with \textsc{DP-SGD}, is therefore:
    \[
        \mathcal{L}(\hat{P}) = \mathbb{E}_{(s,a) \sim \rho^{\pi^B}_P, s^\prime \sim P(\cdot \vert s,a)} \left[- \log \hat{P}(s^\prime \vert s, a) \right] \enspace ,
    \]
    where $\rho^{\pi^B}_P$ is the (normalized) state-action occupancy measure under policy $\pi^B$ and dynamics $P$.
    
    Let us further assume that the true model $P$ belongs to the function class $\mathcal{F}$, and that $P \in \text{argmin}_{P^\prime \in \mathcal{F}} \mathcal{L}(P^\prime)$. We can therefore write the excess population risk of the model estimator $\hat{P}$ as:
    \[
         \mathcal{L}(\hat{P}) -  \mathcal{L}(P) = \mathbb{E}_{(s,a) \sim \rho^{\pi^B}_P, s^\prime \sim P(\cdot \vert s,a)} \left[\frac{\log P(s^\prime \vert s, a)}{\log \hat{P}(s^\prime \vert s, a)} \right] \enspace .
    \]
    But, denoting $D_{\text{KL}}(A, B)$ the Kullback-Leibler divergence between distributions $A, B$:
    \[
        D_{\text{KL}} \left(P(s,a), \hat{P}(s,a)\right) = \mathbb{E}_{s^\prime \sim P(\cdot \vert s,a)} \left[\frac{\log P(s^\prime \vert s, a)}{\log \hat{P}(s^\prime \vert s, a)} \right] \enspace .
    \]
    We can therefore rewrite the above excess population risk as:
    \begin{equation} \label{eq:excess_pop_kl_div}
        \mathcal{L}(\hat{P}) -  \mathcal{L}(P) = \mathbb{E}_{(s,a) \sim \rho^{\pi^B}_P} \left[ D_{\text{KL}} \left(P(s,a), \hat{P}(s,a)\right) \right] \enspace .
    \end{equation}

    If the objective function $\mathcal{L}$ is $L$-Lipschitz and $\Delta$-strongly convex, \cite{BassilyST14} shows (Theorem F.2) that a noisy gradient descent algorithm with $(\epsilon, \delta)$-DP guarantees satisfies, with probability at least $1-\alpha$:
    \begin{equation} \label{eq:bassily_bound}
        \mathcal{L}(\hat{P}) -  \mathcal{L}(P) = \mathcal{O} \left(\frac{L^2 \sqrt{d} \log^2(N / \delta) \cdot \text{poly} \log(1 / \alpha)}{\Delta N \epsilon \alpha}\right) \enspace .
    \end{equation}

    In the non-private case, \cite{ShalevSchwartz2009} provides the following bound under the same assumptions:
      \begin{equation} \label{eq:shalev_bound}
        \mathcal{L}(\hat{P}) -  \mathcal{L}(P) = \mathcal{O} \left(\frac{L^2 \log(N / \alpha)}{\Delta N}\right) \enspace .
    \end{equation}

    %If the objective function $\mathcal{L}$ is Lipschitz, smooth, and satisfies Polyak-Lojasiewicz conditions (assuming no convexity), \cite{KangLLW23} shows (Theorem 3) that the expected excess population risk of \textsc{DP-SGD} satisfies:
    %\begin{equation} \label{eq:kang_bound}
    %    \mathcal{E} \left[\mathcal{L}(\hat{P}) -  \mathcal{L}(P)\right] = \mathcal{O} \left(\min \left\{\frac{d}{N^2 \epsilon^2} + \frac{1}{N}, \frac{\sqrt{d}}{N^{1.5} \epsilon} + \frac{d^{1.5}}{N^{2.5} \epsilon^3}\right\}\right) \enspace .
    %\end{equation}
    
    On the other hand, we have from the Simulation Lemma \citep{Kearns2002, Xu2020} that for a MDP $\mathcal{M}$ with reward upper bounded by $r_{\max} = 1$ and dynamics $P$, a behavioral policy $\pi^B$ and a learn transition model $\hat{P}$ with:
    \begin{equation} \label{eq:model_ineq}
        \mathbb{E}_{(s,a) \sim \rho^{\hat{\pi}_D}_P} \left[ D_{\text{KL}} \left(P(s,a), \hat{P}(s,a)\right) \right] \le \varepsilon_M \enspace ,
    \end{equation}
    which by \ref{eq:excess_pop_kl_div} is equivalent to:
    \begin{equation} \label{eq:model_ineq_loss}
        \mathcal{L}(\hat{P}) -  \mathcal{L}(P) \le \varepsilon_M \enspace ,
    \end{equation}
    
    if the divergence between $\hat{\pi}$ and the behavioral policy is bounded:
    \begin{equation} \label{eq:behav_ineq}
        \max_s D_\text{KL} \left(\hat{\pi}(\cdot \vert s), \pi^B(\cdot \vert s)\right) \le \varepsilon_B \enspace ,
    \end{equation}
    then the value evaluation error of $\hat{\pi}$ is bounded as:
    \begin{equation} \label{eq:value_gap}
        \vert \hat{V}^{\hat{\pi}} - V^{\hat{\pi}} \vert \le \frac{\sqrt{2} \gamma}{(1 - \gamma)^2} \sqrt{\varepsilon_M} + \frac{2 \sqrt{2} \gamma}{(1 - \gamma)^2} \sqrt{\varepsilon_B} \enspace .
    \end{equation}

    Since $f(x) = \mathcal{O}(g(x))$ implies $\sqrt{f(x)} = \mathcal{O}(\sqrt{g(x)})$ \footnote{Indeed, for $f(x)$ positive, for any $x \ge x_0$, $\vert f(x) \vert = f(x) \le M^\prime \times g(x)$, then, for any $x \ge x_0$, $\sqrt{f(x)} = \vert \sqrt{f(x)} \vert \le \sqrt{M^\prime} \times \sqrt{g(x)} = M \times \sqrt{g(x)}$}, we note that we can replace $\sqrt{\varepsilon_M}$ in the model term of the right-hand side of \ref{eq:value_gap} by the (square root of) the bounds from \ref{eq:bassily_bound} and \ref{eq:shalev_bound} in the private case and in the non-private case, respectively.
     
    We must also analyze the policy term of the right-hand side of \ref{eq:value_gap}, which scales with the divergence between $\pi$ and $\pi^B$. On the one hand, privacy aims at preventing overfitting particular data points, and is therefore likely to drive $\hat{\pi}$ away from the offline data,  hence increasing the divergence between $\hat{\pi}$ and $\pi^B$. On the other hand, the reward penalization used in model-based methods like \textsc{MOPO} implicitly constrains $\hat{\pi}$ to stay close to the data collection policy $\pi^B$, as it penalizes the reward on regions of the state-action space where the model uncertainty is high, corresponding to the regions not covered by the offline dataset. This is likely to cancel the diverging effect of privacy. Therefore, we consider that $ \max_s D_\text{KL} \left(\pi(\cdot \vert s), \pi^B(\cdot \vert s)\right)$ will not change significantly due to privacy, and assume that $\varepsilon_B$ stays constant between the private and the non-private case. 
    
    %\textcolor{red}{
    %So, assuming that \ref{eq:behav_ineq} holds for the policy $\pi$ learned with \textsc{PriMORL}, we can plug either \ref{eq:bassily_bound} or \ref{eq:kang_bound} in \ref{eq:model_ineq}.
    %What to do with this result? Should we compare to the non-private case \ref{eq:shalev_bound}? How to deal with expectation, various assumptions etc.}
    
\end{proof}

%%%%%%%%%%%%%%%%%%%%%%%%%%%%%%%%%%%%%%%%%%%%%%%%%%%%%%%%%%%%%%%%%%

%%%%%%%%%%%%%%%%%%%%%%%%%%%%%%%%%%%%%%%%%%%%%%%%%%%%%%%%%%%%%%%%%%

\newpage
\section{Related Work (Extended)}
\label{sec:related_work_ext}

\subsection{Model-Based Offline Reinforcement Learning}

Unlike classical RL \citep{SuttonB98} which is online in nature, offline RL \citep{Levine_Offline_RL_2020, Figueirido2022} aims at learning and controlling autonomous agents without further interactions with the system. This approach is preferred or even unavoidable in situations where data collection is impractical (see for instance \citet{SinghKSRobotic22, SiqiHealthcare2020, autonomous_driving_survey}). Model-based RL \citep{MoerlandMBRLSurvey2023} can also help when data collection is expensive or unsafe as a good model of the environment can generalize beyond in-distribution trajectories and allow simulations. Moreover, model-based RL has been shown to be generally more sample efficient than model-free RL \citep{ChuaCML18}. \citet{ArgensonD21} also show that model-based offline planning, where the model is learned offline on a static dataset and subsequently used for control without further accessing the system, is a viable approach to control agents on robotic-like tasks with good performance. Unfortunately, the offline setting comes with its own major challenges. In particular, when the data is entirely collected beforehand, we are confronted to the problem of \textit{distribution shift} \citep{FujimotoMP19}: as the logging policy used to collect the training dataset only covers a limited (and potentially small) region of the state-action space, the model can only be trusted in this region, and may be highly inaccurate in other parts of the space. This can lead to a severe decrease in the performance of classic RL methods, particularly in the model-based setting where the acting agent may exploit these inaccuracies in the model
, causing large gap between performances in the true and the learned environment.
\textsc{MOPO} \citep{yu_mopo_2020} and \textsc{MOReL} \citep{kidambi_morel_2021}, and more recently \textsc{Count-MORL} \citep{kim_count_2023} have effectively tackled this issue by penalizing the reward proportionally to the model's uncertainty, achieving impressive results on popular offline benchmarks. Nonetheless, there remain many areas for improvement, as highlighted by \cite{mbrl_uncertainty_2022}, which extensively study and challenge key design choices in offline MBRL algorithms.

\subsection{Privacy in Reinforcement Learning}

Differential Privacy (DP), first formalized in \citet{Dwork06}, has become the gold standard in terms of privacy protection. Over the recent years, the design of algorithms with better privacy-utility trade-offs has been a major line of research. In particular, relaxations of differential privacy and more advanced composition tools have allowed tighter analysis of privacy bounds \citep{DworkRV10, DworkR16, BunCDP2016, MironovRDP2017}. Leveraging these advances, the introduction of \textsc{DP-SGD} \citep{Abadi_2016} has allowed to design private deep learning algorithms, paving the way towards a wider adoption of DP in real-world settings, although the practicalities of differential privacy remain challenging
\citep{ponomareva_how_2023}.
In parallel to the theoretical analysis of privacy, many works have focused on designing more and more sophisticated attacks, justifying further the need to design DP algorithms (\citep{RigakiAttacksSurvey2020}). 

Recent works on RL-specific attacks \citep{pan_how_2019, prakash_how_2022, gomrokchi_membership_2022} have demonstrated that reinforcement learning (RL) is no more  immune to privacy threats.
With RL being increasingly used to provide personalized services \citep{Rl_perso_survey_2020}, which may expose sensitive user data, developing privacy-preserving techniques for training policies has become crucial.
Shortly after DP was successfully extended to multi-armed bandits \citep{TossouD16, Basu2019}, a substantial body of work (\textit{e.g.}, \citet{vietri2020private, GarcelonPPP21, Liao2021, luyo2021differentially, chowdhury_differentially_2021, zhou_differentially_2022, ngo_linear_2022, qiao_tabular_2023}) addressed privacy in online RL, extending definitions from bandits. 
However, relying on count-based and UCB-like methods, current RL algorithms with formal DP guarantees are essentially limited to episodic tabular or linear MDPs, and have not been assessed empirically beyond simple numerical simulations.
However, current RL algorithms with formal DP guarantees are essentially limited to episodic tabular or linear MDPs, and have not been assessed empirically beyond simple numerical simulations.
Few works have proposed private RL methods for more general problems, however with significant limitations or in different contexts. \cite{Wang2019} tackle continuous state spaces by adding functional noise to Q-Learning, but the approach is restricted to unidimensional states and focuses on protecting reward information. Recently, \cite{mutual_info_rl_2024} addressed high-dimensional control and robotic tasks; however, they consider a specific notion of privacy that protects sensitive state variables based on a mutual information framework.

Despite the relevance of the setting for real-world RL deployments, private offline RL has received comparatively less attention. To date, only \citet{qiao_offline_2022} have proposed DP offline algorithms, building on non-private value iteration methods. While their approach lays the groundwork for private offline RL and offers strong theoretical guarantees, it remains limited to episodic tabular and linear MDPs. Consequently, no existing work has introduced DP methods that can handle deep RL environments in the infinite-horizon discounted setting, a critical step toward deploying private RL algorithms in real-world applications. With this work, we aim to fill this gap by proposing a differentially private, deep model-based RL method for the offline setting.

%%%%%%%%%%%%%%%%%%%%%%%%%%%%%%%%%%%%%%%%%%%%%%%%%%%%%%%%%%%%%%%%%%%
%%%%%%%%%%%%%%%%%%%%%%%%%%%%%%%%%%%%%%%%%%%%%%%%%%%%%%%%%%%%%%%%%%%

\newpage
\section{Presentation of the Tasks}

%\begin{figure}[t]
%\centering
%\includegraphics[width=0.35\linewidth]{images/cartpole.PNG}
%\caption{The \textsc{Cartpole-Swingup} task.}
%\label{fig:cartpole}
%\end{figure}

\textsc{CartPole} requires to swing up then balance an unactuated pole by applying forces on a cart at its base, while \textsc{CartPole-Balance} only requires keeping balance. The duration of both tasks is 1,000 steps.
\textsc{Pendulum} involves controlling an inverted pendulum by applying torque to keep it upright and balanced over 200 steps. For this task, we normalize the episodic return to obtain a normalized score between 0 and 1000, using the following formula: $s_{\text{normalized}} = \frac{s - (-1500)}{0 - (-1500))}$.
% during 10-second (or 1,000-step) episodes. 
%The observation space $\mathcal{S} \subset \mathbb{R}^5$ and the action space $\mathcal{A} = [-1, 1]$ are both continuous and have 5 dimensions and 1 dimension, respectively. %This task is usually considered solved when the episodic return exceeds 800.
% Talk about "good performance" ?
\textsc{HalfCheetah} is another, higher-dimensional, continuous control task from OpenAI's Gym \citep{openai_gym} based on the physics engine MuJoCo \citep{mujoco} where we move forward a 2D cat-like robot by applying torques on its joints. The duration of an episode is 1,000 steps. 
% in 1,000-step episodes. 
%Its observation space $\mathcal{S} \subset \mathbb{R}^{17}$ has 17 dimensions and its action space has $\mathcal{A} \subset [0,1]^6$ has 6 dimensions.

\section{Data Collection}
\label{sec:app_data_collection}

To collect our offline dataset for \textsc{CartPole} and \textsc{Pendulum}, we used DDPG \citep{DDPG}, a model-free RL algorithm for continuous action spaces. We ran 600 independent runs of $50,000$ steps each for \textsc{Cartpole-Balance}, 150 independent runs of $200,000$ steps each for \textsc{Cartpole-SwingUp}, and 6 independent runs of $1M$ steps each for \textsc{Pendulum}. We collect all training episodes to ensure a correct mix between random, medium and expert episodes (similar to \textit{replay} datasets in \cite{fu2020d4rl}).

\section{Baselines}

\begin{table}[h]
\caption{\textsc{PriMORL} configurations.}
\label{tab:dp_morl_config}
\vskip 0.15in
\begin{center}
\begin{small}
\begin{sc}
\begin{tabular}{lcccc}
\toprule
Variant & \textsc{Trajectory-Level Ens. Training} & Clip & Noise & DP \\
\midrule
No Clip  & $\cmark$ & $\xmark$ & $\xmark$ & $\epsilon = \infty$\\
No Privacy  & $\cmark$ & $\cmark$ & $\xmark$ & $\epsilon = \infty$ \\
Low, High &  $\cmark$ & $\cmark$ & $\cmark$ & $\epsilon < \infty$ \\
%Medium Privacy  &  $\cmark$ & $\cmark$ & Medium & $\epsilon < \infty$ \\
%High Privacy  &  $\cmark$ & $\cmark$ & High & $\epsilon < \infty$ \\
\bottomrule
\end{tabular}
\end{sc}
\end{small}
\end{center}
\vskip -0.1in
\end{table}

The first two baselines, \textsc{PriMORL No Clip} and \textsc{PriMORL No Privacy} are not private ($\epsilon < \infty$) but allow us to isolate the impact of trajectory-level model ensemble training (without clipping and noise addition) and clipping on policy performance. We do not report results for \textsc{PriMORL No Clip} for \textsc{CartPole} and \textsc{Pendulum} as we found that the model optimized with \textsc{TDP Model Ensemble Training} diverges without clipping.

\section{Experiment Details} \label{sec:detail_exps}

\subsection{Datasets}

Table~\ref{tab:datasets} provides additional details on the offline datasets used in experiments. Figure~\ref{fig:dataset_statistics} shows episode return statistics for each dataset.

\begin{table}[ht]
\caption{Dataset details}
\label{tab:datasets}
\vskip 0.15in
\begin{center}
\begin{small}
\begin{sc}
\begin{tabular}{lccc}
\toprule
 & \textsc{CartPole} & \textsc{Pendulum} & \textsc{HalfCheetah} \\
\midrule
Origin & Custom & Custom & \textsc{D4RL} \\
Observation space $\mathcal{S}$ & $\mathbb{R}^5$ & $\mathbb{R}^3$ & $\mathbb{R}^{17}$ \\
Action space $\mathcal{A}$ & $[-1,1]$ & $[-2,2]$ & $[0,1]^6$ \\
Nb. of episodes $K$ & 30,000 & 30,000 & 2,003 \\
\bottomrule
\end{tabular}
\end{sc}
\end{small}
\end{center}
\vskip -0.1in
\end{table}

\begin{figure}[t]
\label{fig:dataset_statistics}
\centering
%\vspace{0pt}
\includegraphics[width=0.8\linewidth]{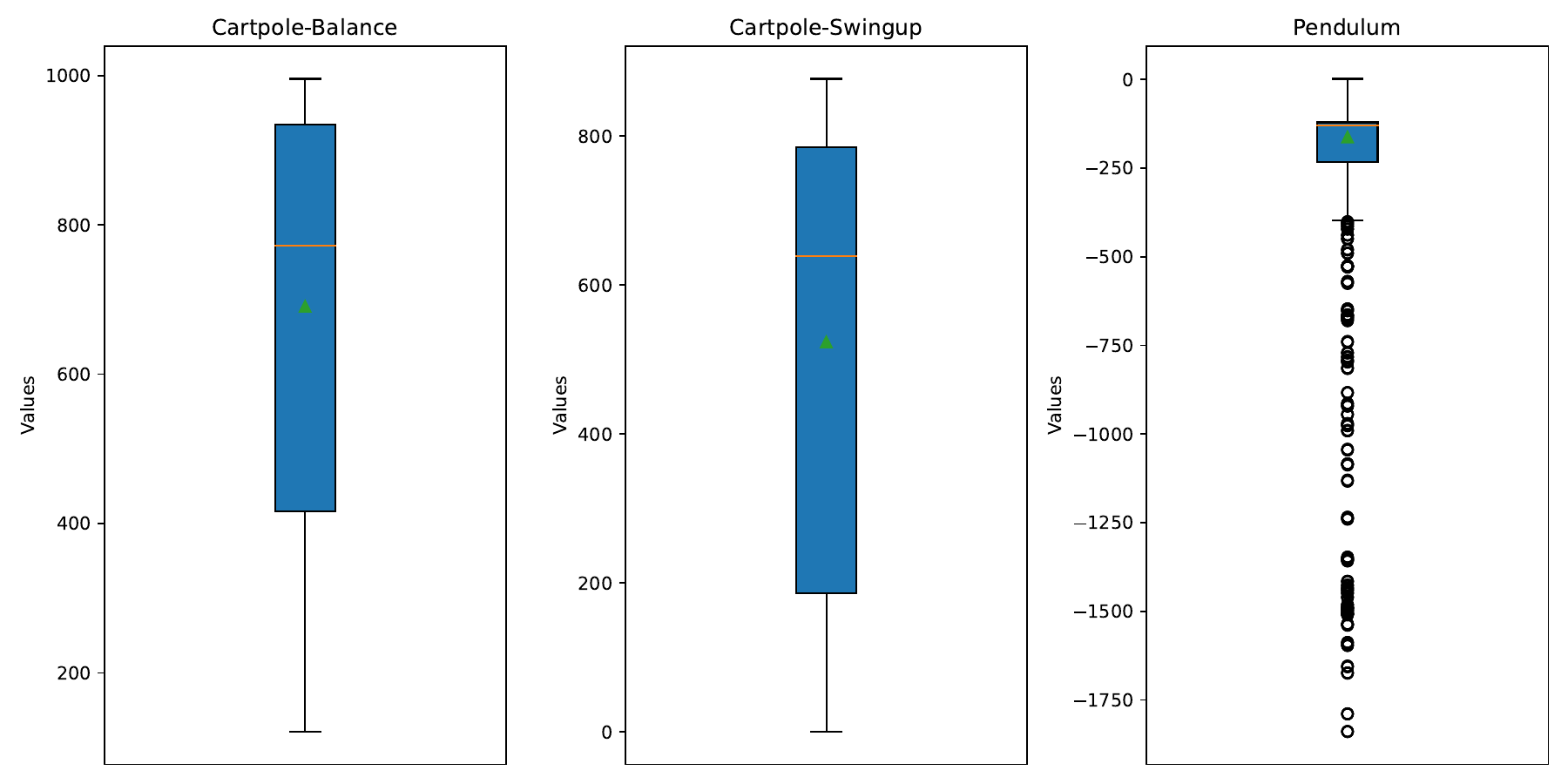}
\end{figure}

\subsection{Implementation Details}

For all tasks, the model is approximated with a deep neural network with \textsc{Swish} activation functions and decaying weights. Models take as input a concatenation of the current state $s$ and the taken action $a$ and predict the difference between the next state $s^\prime$ and the current state $s$ along with the reward $r$. Table \ref{tab:implementation} provides further implementation details. 

The code repository for \textsc{PriMORL} is provided as part of the supplementary material and will be made public upon acceptance. For \textsc{MOPO}, we use the official implementation from \url{https://github.com/tianheyu927/mopo}, as well as the PyTorch re-implementation from \url{https://github.com/junming-yang/mopo}. Our implementation of \textsc{PriMORL}, which mainly uses PyTorch, is also based on these codebases. To collect the datasets, we use DDPG implementation from \url{https://github.com/schatty/DDPG-pytorch}.

Model training with \textsc{TDP Model Ensemble Training} is parallelized over 16 CPUs using JobLib, while SAC training is conducted over a single Nvidia Tesla P100 GPU.

\begin{table}[ht]
\caption{Implementation details}
\label{tab:implementation}
\vskip 0.15in
\begin{center}
\begin{small}
\begin{sc}
\begin{tabular}{lccc}
\toprule
 & \textsc{CartPole} & \textsc{Pendulum} & \textsc{HalfCheetah} \\
\midrule
Model input dimension & 6 & 4 & 23 \\
Model output dimension & 6 & 4 & 18 \\
Model hidden layers & 2 & 2 & 4 \\
Neurons per layer & 128 & 64 & 200 \\
Weight decay & \cmark & \cmark & \cmark \\
Activation functions & SWISH & SWISH & SWISH \\
Ensemble size $N$ & 5 & 3 & 7 \\
\bottomrule
\end{tabular}
\end{sc}
\end{small}
\end{center}
\vskip -0.1in
\end{table}

\subsection{Training Details}

\begin{table}[ht]
\caption{Training and Hyperparameters details}
\label{tab:training_hyper}
\vskip 0.15in
\begin{center}
\begin{small}
\begin{sc}
\begin{tabular}{lccc}
\toprule
 & \textsc{CartPole} & \textsc{Pendulum} & \textsc{HalfCheetah} \\
\midrule
Test set size & $1\% \times K$ & $1\% \times K$ & $10\% \times K$ \\
Early stopping & $\cmark \text{ patience} = 10$ & $\cmark \text{ patience} = 10$ &  $\cmark \text{ patience} = 5$ \\
\midrule
Sampling ratio $q$ & $10^{-3}$ & $10^{-3}$ & $10^{-2}$ \\
Model local epochs $E$ & 1 & 1 & 1 \\
Model batch size $B$ & 16 & 16 & 16 \\
Model LR $\eta$ & $10^{-3}$ & $10^{-3}$ & $10^{-3}$ \\
Clipping Strategy & Flat & Per-Layer & Per-Layer \\
\midrule
SAC LR & $3.10^{-4}$ & $3.10^{-4}$ & $3.10^{-4}$\\ 
Rollout length $H$ & 20 & 30 & 5 \\
Reward penalty $\lambda$ & 2.0 & 2.0 & 1.0 \\
Auto-$\alpha$ & $\cmark$ & $\cmark$ & $\cmark$ \\
Target entropy $H$ & -3 & -3 & -3 \\
\bottomrule
\end{tabular}
\end{sc}
\end{small}
\end{center}
\vskip -0.1in
\end{table}

Before model training, we split the offline dataset into two parts: a train set used to train the model, and a test set used to track model performance. We consider the test set public so that this operation does not involve additional privacy leakage. The split is made by episode (instead of by transitions), so that the test set contains 1\% of the episodes for \textsc{CartPole} and \textsc{Pendulum} and 20\% for \textsc{HalfCheetah}. To tune the clipping norm, we set $z=0$ and progressively decreased $C$ until it started to adversely affect performance provided the best results. Moreover, we set the sampling ratio so that a few dozen episodes are randomly selected at each step, which proved to work best in our experiments, which correspond to $q=10^{-3}$ for \textsc{Cartpole} and \textsc{Pendulum}. The model is trained until convergence using \textit{early stopping}. Test set prediction error is used to track model improvement. For SAC training, the real-to-model ratio $r_{\text{real}}$ is zero, meaning that SAC is trained using only simulated data from the model, and does not access any data from the offline dataset. Training details are provided in Table~\ref{tab:training_hyper}.

\subsection{Hyperparameters}

The model is trained using \textsc{TDP Model Ensemble Training} with learning rate $\eta=10^{-3}$, batch size $B=16$, and number of local epochs $E=1$. 

The policy is optimized within the model using Soft Actor-Critic with rollout, with rollout length and penalty depending on the task. We use a learning rate of $3 . 10^{-4}$ for both the actor and the critic. For entropy regularization, we use auto-$\alpha$ with target entropy $H=-3$.

Hyperparameters are summarized in Table~\ref{tab:training_hyper}. We do not report the privacy loss resulting from hyperparameter tuning, although we recognize its importance in real-world applications.

\subsection{Privacy Parameters}

In Table \ref{tab:cartpole_results}, we provide the privacy budgets $\epsilon$ computed with the moments accountant method from \citet{Abadi_2016}. We use the DP accounting tools from Google's Differential Privacy library, availabe on \href{https://github.com/google/differential-privacy}{GitHub}. Privacy budget are computed for $\delta=10^{-5}$, \textit{i.e.} less than $K^{-1}$ as recommended in the literature. It also depends on the noise multiplier $z$, the number of training round $T$ and the sampling ratio $q$. Since we use early stopping and the different training runs have different durations, we use the average number of training rounds in the privacy budget computations.

For \textsc{CartPole-Balance}, we use $z=0.25$ and $z=0.45$ for \textsc{PriMORL Low} and \textsc{PriMORL High}, respectively. For \textsc{CartPole-SwingUp}, we use $z=0.25$ and $z=0.38$ for \textsc{PriMORL Low} and \textsc{PriMORL High}, respectively. The value for \textsc{PriMORL High} is chosen by incrementally increasing $z$ until policy performance drops below acceptable levels. The corresponding $\epsilon$ is therefore roughly the best privacy budget we can obtain while keeping acceptable policy performance. The value for \textsc{PriMORL Low} is chosen arbitrarily to provide a weaker level of privacy that typically yields higher policy performance, illustrating the trade-off between the strength of the privacy guarantee and the performance. 

\subsection{Computational Resources} \label{sec:comp_resources}

We perform training on a single machine with 64 CPUs and 6 Tesla P100 GPUs with 16GB RAM each. The full training of a single policy, from model learning to policy optimization, takes several hours.

\section{Additional Experiments}

Figures~\ref{fig:diff_estimators} and \ref{fig:pendulum_f_eps} provide additional empirical insights about \textsc{PriMORL}. In Figure~\ref{fig:diff_estimators}, we can see that neither uncertainty estimators is superior overall, but the choice of estimator may impact privacy performance for a specific tasks. On the other hand, Figure~\ref{fig:pendulum_f_eps} shows the performance of \textsc{PriMORL} on \textsc{Pendulum} as a function of the privacy strength $\epsilon$. We can see that performance does not degrade until $\epsilon$ goes in the $1$ to $10$ range, where it starts to drop significantly.

\begin{figure*}[t]
\centering
%\vspace{0pt}
\includegraphics[width=0.65\linewidth]{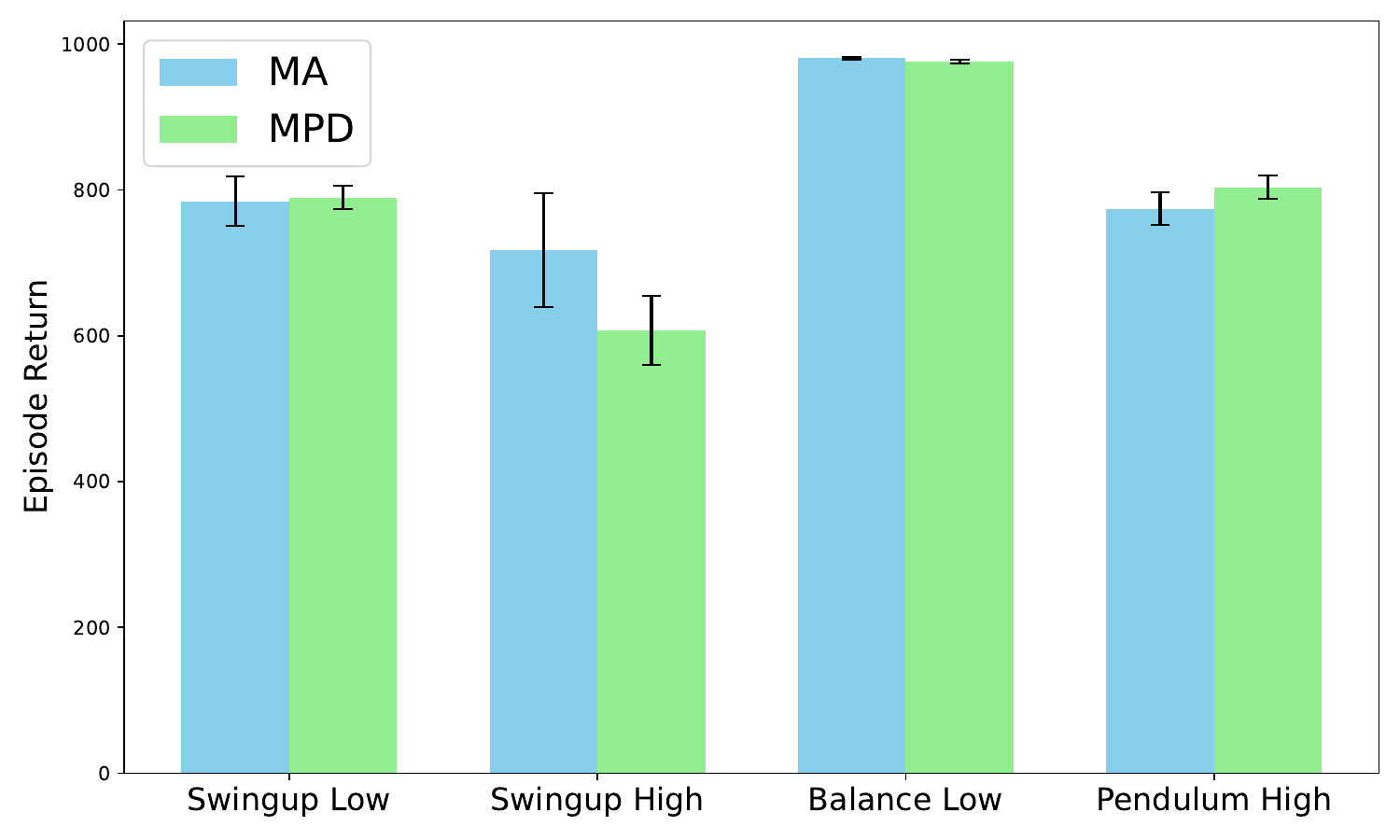}
\caption{Comparison of policy performance with $u_\text{MA}$ and $u_\text{MPD}$ for a fixed model. We measure the average performance of the policy over the last 10 epochs of training. Average and confidence intervals are computed over 5 random seeds.}
\label{fig:diff_estimators}
\end{figure*}

\begin{figure*}[t]
\centering
%\vspace{0pt}
\includegraphics[width=0.65\linewidth]{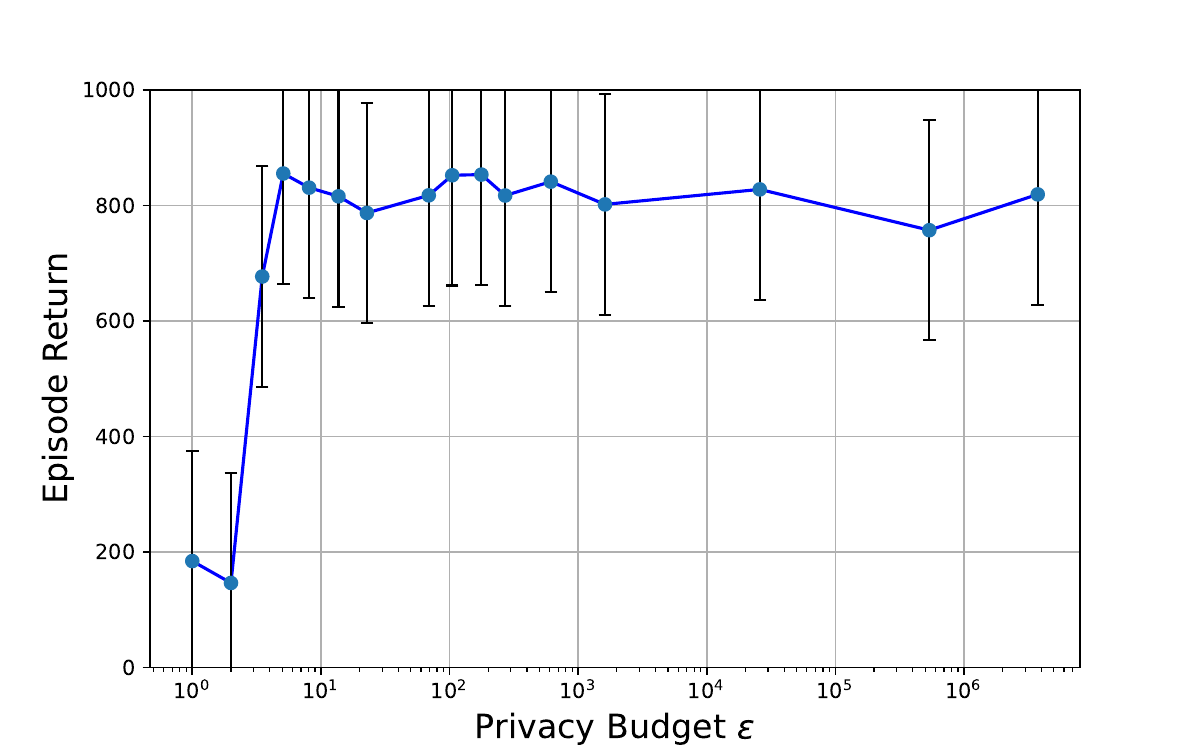}
\caption{Policy performance on \textsc{Pendulum} as a function of the privacy budget $\epsilon$. We measure the average performance of the policy over the last 5 epochs of training. Average and confidence intervals are computed over 5 random seeds.}
\label{fig:pendulum_f_eps}
\end{figure*}

\section{Experiments on \textsc{HalfCheetah}} \label{sec:exps_halfcheetah}

We conduct experiments on the \textsc{Medium-Expert} dataset ($K=2,003$) from the classic \textsc{D4RL} benchmark \citep{fu2020d4rl}. %, where agents are anticipated to achieve peak performance, facilitating the analysis of privacy cost.
%with a mixture of expert-performance trajectories generated by a fully-trained policy and medium trajectories generated by a half-trained policy. \albert{j'ai oublie, est-ce qu'il y a une raison pour laquelle on prend ce dataset la? si oui peut etre la dire?}
%which is a mixture of expert-performance and medium-performance trajectories. Indeed, this dataset is where our offline agents are anticipated to achieve peak performance, facilitating the analysis of the privacy cost.
Experimental results are reported in Figure~\ref{fig:halfcheetah_plot} and Table~\ref{tab:halfcheetah_results} (in appendix), using $C=15.0$ and $q=10^{-2}$. 

\begin{figure*}[t]
\centering
%\vspace{0pt}
\includegraphics[width=0.65\linewidth]{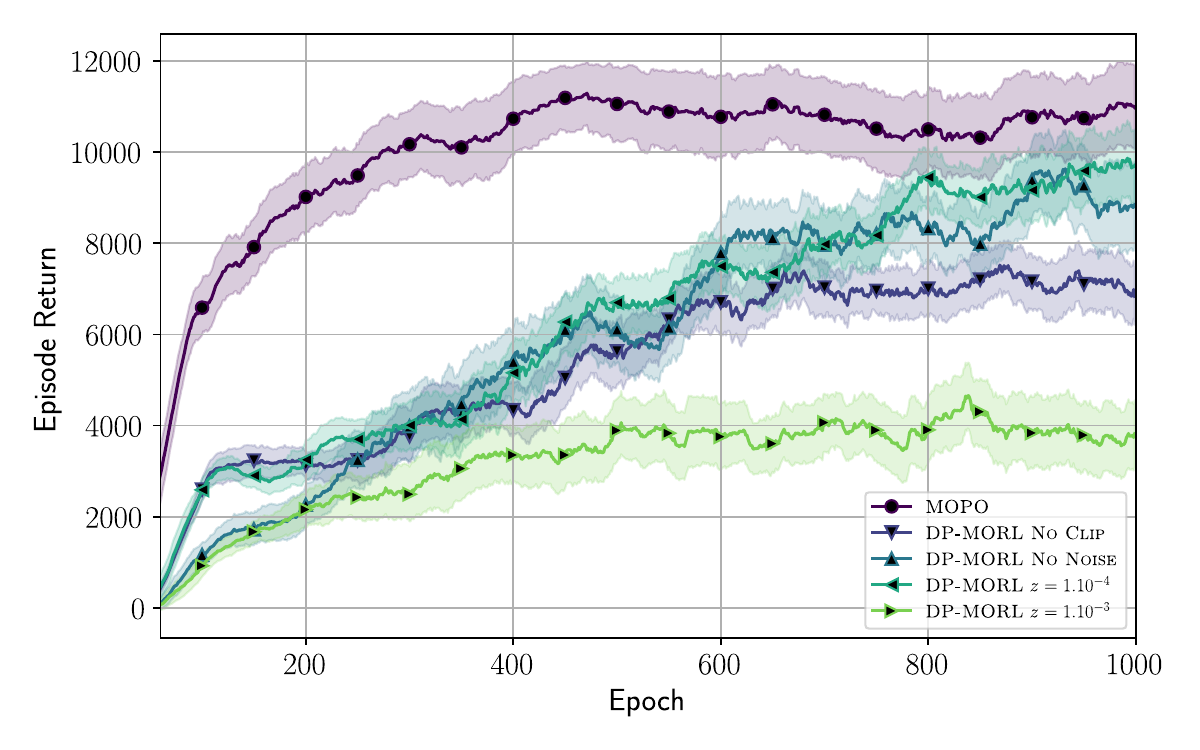}

\caption{Learning curves for the SAC policy on \textsc{HalfCheetah} (\textit{right}). Policy performance (episodic return) is evaluated in the true MDP at the end of each training epoch, over 10 evaluation episodes with different random seeds.}
\label{fig:halfcheetah_plot}
\end{figure*}

\begin{table}[ht]
\caption{Results for \textsc{HalfCheetah Medium-Expert}. \textsc{Return} is the return of the SAC policy evaluated over 10 episodes at the end of each training epoch, averaged over the last 20 epochs.}
\label{tab:halfcheetah_results}
\vskip 0.15in
\begin{center}
\begin{small}
\begin{sc}
\begin{tabular}{lcr}
\toprule
Method & $z$ & Return \\
\midrule
MOPO & 0.0 & 10931 $\pm$ 1326 \\
\midrule
PriMORL No Clip  & 0.0 & 7062 $\pm$ 2230 \\
PriMORL No Noise  & 0.0 & 8792 $\pm$ 2053 \\
\midrule
PriMORL &  $z=1.10^{-4}$ & 9729 $\pm$ 2018 \\
 &  $z=1.10^{-3}$ & 3697 $\pm$ 1465 \\
\bottomrule
\end{tabular}
\end{sc}
\end{small}
\end{center}
\vskip -0.1in
\end{table}

If \textsc{PriMORL} can train competitive policies with small enough noise levels --- a tiny amount of noise like $z=10^{-4}$ proving even beneficial, possibly acting as a kind of regularization ---,  we were not able to obtain reasonable $\epsilon$'s. Indeed, a noise multiplier as small as $z=10^{-3}$ is enough to cause a significant decline in performance. \textsc{HalfCheetah} thus appears a significantly harder tasks than \textsc{Cartpole} and \textsc{Pendulum}. It is not surprising as $\textsc{HalfCheetah}$ is higher-dimensional, and the theoretical analysis led in Section\ref{sec:impact_privacy_policy_opt} showed that the dimension $d$ of the problem could negatively impact the performance of the policy. However, we point out that the size of the dataset for \textsc{HalfCheetah} is very limited, and argue that larger datasets with substantially more episodes would translate into competitive privacy-performance trade-offs, as we develop in Section \ref{sec:price_of_privacy}.

\newpage
\section{Algorithms} \label{sec:algos}

Algorithm~\ref{alg:dp_model_training_full} is the fully detailed pseudo-code for \textsc{PriMORL}. Algorithm~\ref{alg:per_layer_clipping} details the clipping method used in \textsc{TDP Model Ensemble Training}. Algorithm~\ref{alg:sac_training} is the pseudo-code for SAC policy optimization on the pessimistic private model. This pseudo-code is based on \url{https://spinningup.openai.com/en/latest/algorithms/sac.html}

\begin{algorithm}[ht]
\begin{algorithmic}[1]
\STATE {\bfseries Input:} offline dataset $\mathcal{D}_K$, sampling ratio $q \in (0,1)$, noise multiplier $z \ge 0$, clipping norm $C > 0$, local epochs $E$, batch size $B$, learning rate $\eta$
\STATE {\bfseries Output:} private model $\hat{M}_\theta$
\STATE Initialize model parameters $\theta_0$
\FOR{each iteration $t \in [\![0, T-1]\!]$}
    \STATE $\mathcal{U}_t \leftarrow $ (sample with replacement trajectories from $\mathcal{D}_K$ with prob. $q$)
    \FOR{each trajectory $\tau_k \in \mathcal{U}_t$}
        \STATE Clone current models $\left\{\theta_i^{\text{start}}\right\}_{i=1}^N \leftarrow \left\{\theta_i(t)\right\}_{i=1}^N$
        \STATE $\theta \leftarrow \theta^{\text{start}} :=\left(\theta^{\text{start}}\right)_{i=1}^N$
        \FOR{each local epoch \tikzmark{top} $i \in [\![1, E]\!]$}
            \STATE $\mathcal{B} \leftarrow $ ($\tau_k$'s data split into size $B$ batches)
            \FOR{each batch $b \in \mathcal{B}$}
                \STATE $\theta \leftarrow \theta - \eta \nabla \mathcal{L}(\theta; b)$
                \STATE $\theta \leftarrow \theta^{\text{start}} + \text{\textsc{EnsembleClip}}(\theta - \theta^{\text{start}}; C)$ \tikzmark{right}
            \ENDFOR
        \ENDFOR \tikzmark{bottom}
        \STATE $\Delta^{\text{clipped}}_{t,k} \leftarrow \theta - \theta^{\text{start}}$
    \ENDFOR
    \STATE $\Delta_i^{\text{avg}}(t) = \frac{\sum_{k \in \mathcal{U}_t} \Delta_{i, k}^{\text{clipped}}(t)}{q K}$
    \STATE $\theta(t+1) \leftarrow \theta(t) + \Delta^{\text{avg}}(t) + \mathcal{N}\left(0_{N_d}, \left(\frac{zC}{qK}\right)^2 I_{N_d}\right)$
\ENDFOR
\end{algorithmic}
\AddNote{top}{bottom}{right}{$\textsc{  EnsClipGD}\left(\tau_k, \left\{\theta_i^{\text{start}}\right\}_{i=1}^N; C, E, B\right)$}
\caption{Model Training with \textsc{TDP Model Ensemble Training}}
\label{alg:dp_model_training_full}
\end{algorithm}

\begin{algorithm}[ht]
\begin{algorithmic}[1]
\STATE {\bfseries Input:} ensemble size $N$, number of model layers $L$, unclipped gradient $\Delta = \left\{\Delta_{i, \ell}\right\}_{i, \ell=1}^{N, L}$, clipping norm $C$
\STATE {\bfseries Output:} clipped gradient $\Delta^{\text{clipped}}$
\STATE $\Delta_i \leftarrow \left(\Delta_{i, \ell}\right)_{\ell=1}^L, \enspace C_i = \frac{C}{\sqrt{N}}$
\STATE \[\Delta^{\text{clipped}}_i \leftarrow \frac{\Delta_i}{\max\left(1, \frac{\lVert \Delta_i \rVert}{C_i} \right)}, \enspace j=1,...,m. \]
\end{algorithmic}
\caption{Ensemble Clipping (\textsc{EnsembleClip})}
\label{alg:per_layer_clipping}
\end{algorithm}

\begin{algorithm}[ht]
\begin{algorithmic}[1]
\STATE {\bfseries Input:} private model $\hat{M}=(\hat{P}, \hat{r})$, empty replay buffer $\mathcal{B}$, uncertainty estimator $u \in \{u_\text{MA}, u_\text{MPD}\}$
\STATE {\bfseries Output:} private policy $\hat{\pi}^{\text{DP}}$
\STATE Initialize policy parameters $\xi$, Q-function parameters $ \omega_1, \omega_2$ and target parameters $\omega_{\text{targ}, 1}, \omega_{\text{targ}, 2}$
\FOR{epoch $e \in [\![1, E]\!]$}
    \WHILE{episode is not terminated}
        \STATE Observe state $s$ and select action $a \sim \pi_{\xi}\left(\cdot \vert s\right)$
        \STATE Execute $a$ in the pessimistic MDP $\Tilde{\mathcal{M}}$ and observe next state $s^\prime \sim \hat{P}\left(\cdot \vert s,a\right)$, reward $r \sim \hat{r}(s,a) - \lambda u(s,a)$ and done signal $d$
        \STATE Store $\left(s, a, r, s^\prime, d\right)$ in replay buffer $\mathcal{B}$
        \IF{time to update}
            \STATE Sample a batch of transitions $B = \left\{\left(s, a, r, s^\prime, d\right)\right\}$ from buffer $\mathcal{B}$
            \STATE Compute targets for Q-functions:
            \[
                y(r, s^\prime, d) = r + \gamma(1-d) \left(\min_{i=1,2} Q_{\omega_{\text{targ}, i}}(s^\prime, \Tilde{a}^\prime) - \alpha \log \pi_\xi(\Tilde{a}^\prime \vert s^\prime)\right), \enspace \Tilde{a}^\prime \sim \pi_\xi(\cdot \vert s^\prime) \enspace .
            \]
            \STATE Update Q-functions by one step of gradient descent using:
            \[
                \nabla_{\omega_i} \frac{1}{\vert B \vert} \sum_{(s,a,r,s^\prime,d) \in B} \left(Q_{\omega_i}(s,a) - y(r,s^\prime,d)\right)^2, \enspace \text{for } i=1,2.
            \]
            \STATE Update policy by one step of gradient ascent using:
            \[
                 \nabla_{\xi}\frac{1}{\vert B \vert} \sum_{s \in B} \left(\min_{i=1,2} Q_{\omega_i} (s, \Tilde{a}_\zeta(s)) - \alpha \log \pi_\xi \left(\Tilde{a}_\zeta(s) \vert s\right)\right), \enspace \Tilde{a}_\zeta(s) \sim \pi_\xi(\cdot \vert s).
            \]
            %where $\Tilde{a}_\zeta(s)$ is a sample from $\pi_\xi(\cdot \vert s)$ which is differentiable w.r.t. $\xi$ \textit{via} the reparameterization trick
            \STATE Update target networks with:
            \[
                \omega_{\text{targ}, i} \leftarrow \rho \omega_{\text{targ}, i} + (1 - \rho)\omega_i, \enspace \text{for } i=1,2 \enspace.
            \]
        \ENDIF
    \ENDWHILE
    \STATE Evaluate $\pi_\xi$ is the true environment $\mathcal{M}$.
\ENDFOR
\end{algorithmic}
\caption{Private Model-Based Optimization with SAC}
\label{alg:sac_training}
\end{algorithm}

%%%%%%%%%%%%%%%%%%%%%%%%%%%%%%%%%%%%%%%%%%%%%%%%%%
%%%%%%%%%%%%%%%%%%%%%%%%%%%%%%%%%%%%%%%%%%%%%%%%%%
%%%%%%%%%%%%%%%%%%%%%%%%%%%%%%%%%%%%%%%%%%%%%%%%%%
\newpage

%%%%%%%%%%%%%%%%%%%%%%%%%%%%%%%%%%%%%%%%%%%%%%%%%%
%%%%%%%%%%%%%%%%%%%%%%%%%%%%%%%%%%%%%%%%%%%%%%%%%%
%%%%%%%%%%%%%%%%%%%%%%%%%%%%%%%%%%%%%%%%%%%%%%%%%%
\clearpage
\section{The Price of Privacy in Offline RL} \label{sec:app_price_privacy}

In this section, we provide theoretical and practical arguments to further justify the need for (much) larger datasets in order to achieve competitive privacy trade-offs in offline RL, as pointed out in (Section~\ref{sec:exps}).

\textbf{Why does privacy benefit so much from large datasets?} From a theoretical perspective, it stems from two facts: 1) $\epsilon$ scales with the sampling ratio $q$ (\textit{privacy amplification by subsampling}), and 2)  noise magnitude $\sigma$ is inversely proportional to $\mathbb{E}\left[\lvert \mathcal{U}_t \rvert \right] = qK$. Clearly, the privacy-performance trade-off would benefit from both small $q$ (reducing $\epsilon$) and large $qK$ (reducing noise levels and thus improving performance), which are conflicting objectives for a fixed $K$. However, if we consider using larger datasets of size $K^\prime \gg K$, it becomes possible to find a $K^\prime$ large enough so that we can use $q^\prime \ll q$ and $q^\prime K^\prime \gg q K$, achieving both much stronger privacy and better performance. We can even argue that for a given privacy budget $\epsilon$ (obtained for a given $q$) and an unlimited capacity to increase $K$, we could virtually tend to zero noise levels and achieve optimal performance. 
Therefore, \textsc{PriMORL}, already capable of producing good policies with significant noise levels and $\epsilon$, has the potential to achieve stronger privacy guarantees provided access to large enough datasets.

An aspect that deserves further development is the iterative aspect of the used training methods and its effect on privacy. Differential privacy being a worst-case definition, it assumes that all intermediate models are released during training. Although the practicality of this hypothesis is debatable, it definitely impacts privacy: privacy loss is incurred at each training iteration (corresponding to a gradient step on the global model in \textsc{DP-SGD} and \textsc{TDP Model Ensemble Training}) and privacy budget, therefore, scales with the number of iterations $T$. Consequently, limiting the number of iterations is even more crucial with DP training than with non-private training. Training a model on the kind of tasks we considered nonetheless requires a lot of iterations to reach convergence (empirically, thousands of iterations for \textsc{CartPole} and tens of thousands of iterations for \textsc{HalfCheetah}), and the privacy budget suffers unavoidably.

However, one way to circumvent this is to leverage privacy amplification by subsampling. Indeed, as \citet{McMahan_FedAvg_17} observe, the additional privacy loss incurred by additional training iterations becomes negligible when the sampling ratio $q$ is small enough, which is a direct effect of privacy amplification by subsampling. We discussed above how increasing dataset size $K$ allowed to decrease both sampling ratio $q$ and noise levels. Therefore, by increasing the size of the dataset, we also greatly reduce the impact of the number of training iterations, likely promoting model convergence. This further reinforces the need for large datasets in offline RL in order to study privacy. As an example, \citet{McMahanRT018} consider datasets with $10^6$ to $10^9$ users to train DP recurrent language models, and this is arguably the main reason why they achieve formal strong privacy guarantees. For comparison, the classical \textsc{RL Unplugged} and \textsc{D4RL} benchmarks provide datasets with $K \approx 10^1$ to $K \approx 10^3$ datasets. Achieving the privacy-performance trade-offs demonstrated in Section~\ref{sec:exps} would not have been possible without the collection of large datasets. Moreover, datasets orders of magnitude larger would be required to attain formal, strong privacy guarantees, such as $\epsilon < 1$. While conducting experiments in deep offline RL with such extensive datasets demands substantial computational resources, we argue that scenarios involving access to datasets with a vast number of trajectories are reflective of real-world situations. For this reason, we consider this case worthy of thorough investigation.

Figure~\ref{fig:epsilon_plots} illustrates this point in another way. Given $\epsilon \in \{10^{-4}, 10^{-3}, 10^{-2}\}$, we plot for a range of sampling ratio $q$ the maximum number of iterations $T$ that is allowed so that the total privacy loss does not exceed $\epsilon$, as a function of the noise multiplier $z$. We can see how decreasing $q$ makes it well easier to train a private model: dividing $q$ by 10, we "gain" roughly 10 times more iterations across all noise levels.

\begin{figure*}[ht]
\centering
\begin{minipage}{0.49\textwidth}
%\vspace{0pt}
\includegraphics[width=\linewidth]{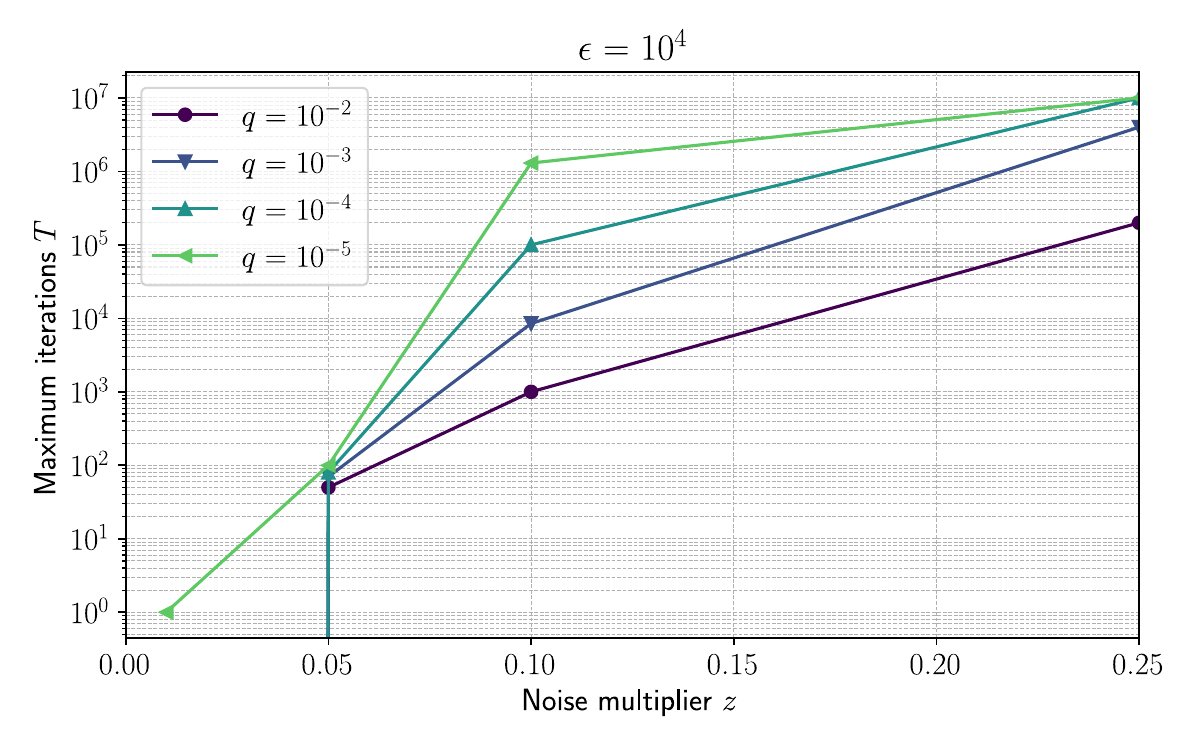}
\end{minipage}
\begin{minipage}{0.49\textwidth}
%\vspace{0pt}
\includegraphics[width=\linewidth]{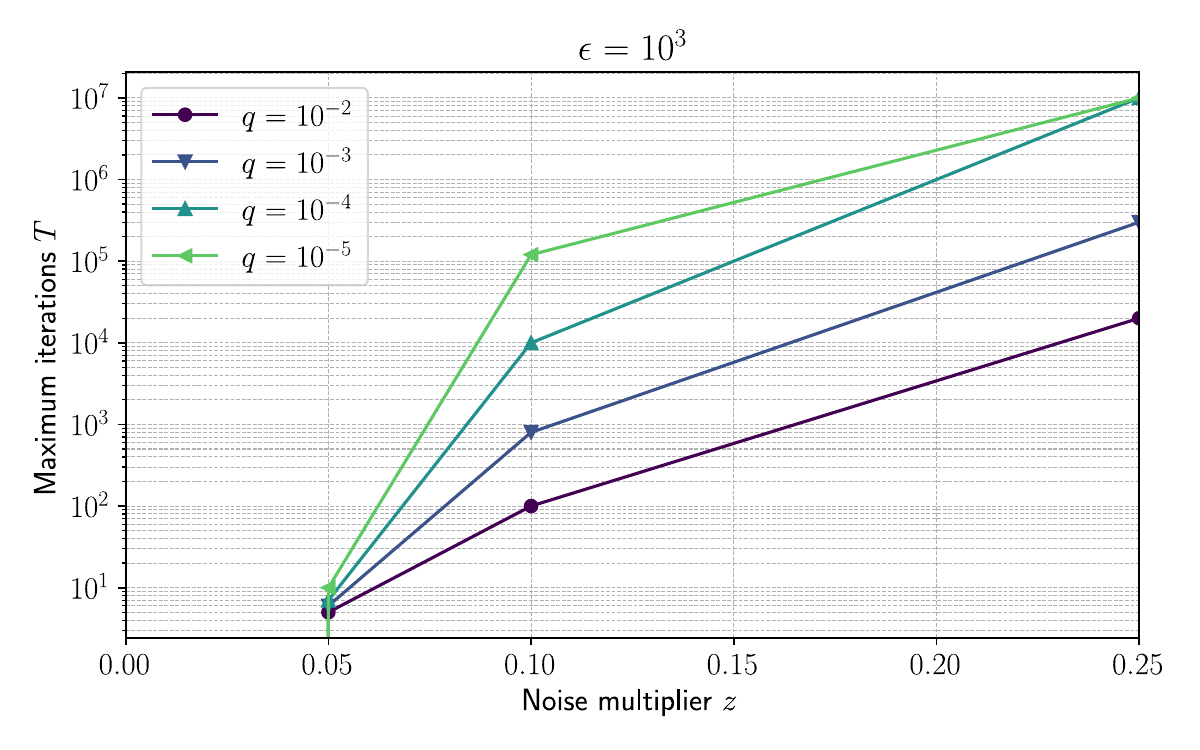}
\end{minipage}
\hfill
\begin{minipage}{0.49\textwidth}
%\vspace{0pt}
\includegraphics[width=\linewidth]{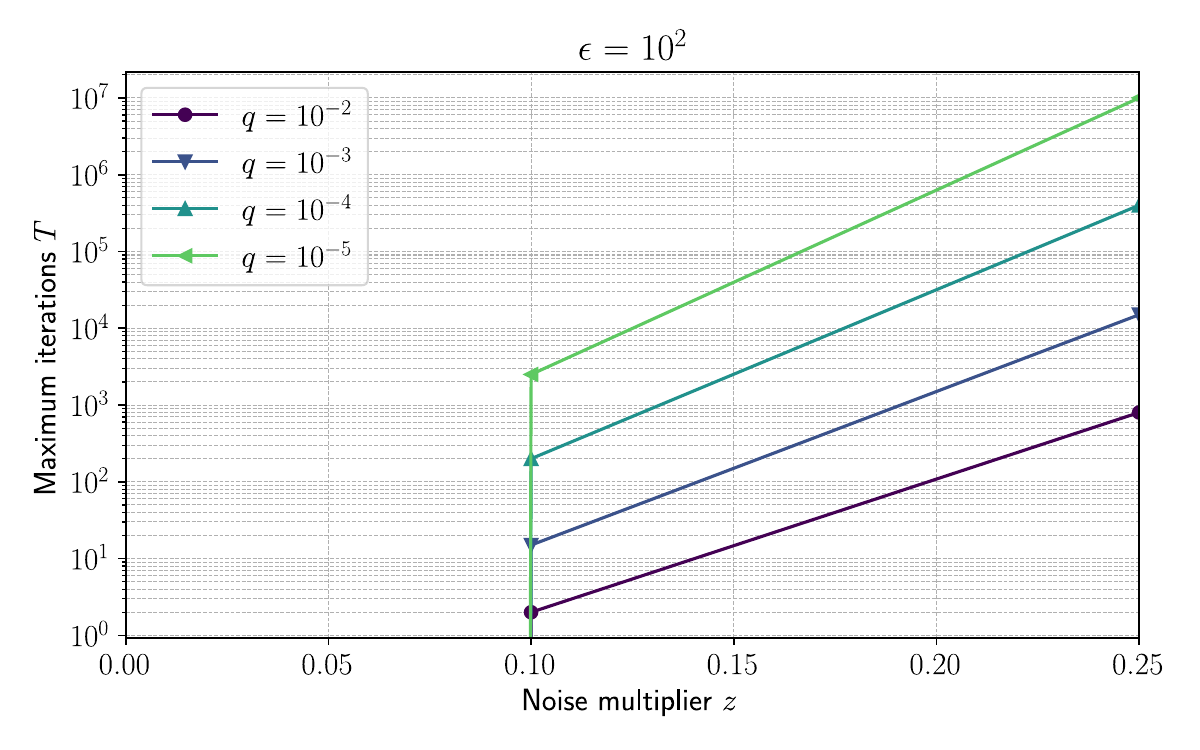}
\end{minipage}
\vspace{-2pt}
\caption{Maximum number of iterations $T$ so that the privacy loss does not exceed $\epsilon$, as function of the noise multiplier $z$.}
\label{fig:epsilon_plots}
\end{figure*}

%%%%%%%%%%%%%%%%%%%%%%%%%%%%%%%%%%%%%%%%%%%

\newpage 
\section{Broader Impacts} \label{sec:broader_impacts}

As recent advances in the field have moved reinforcement learning closer to widespread real-world application, from healthcare to autonomous driving, and as many works have shown that it is no more immune to privacy attacks than any other area in machine learning, it has become crucial to design algorithmic techniques that protect user privacy. In this paper, we contribute to this endeavor by introducing a new approach to privacy in offline RL, tackling more complex control problems and thus paving the way towards real-world private reinforcement learning. We firmly believe in the importance of pushing the boundaries of this research field and are hopeful that this work will contribute to practical advancements in achieving trustworthy machine learning.

\end{document}